\documentclass{article}
\usepackage{tabularx}

\usepackage{microtype}
\usepackage{graphicx}
\usepackage{subcaption}
\usepackage{booktabs} 
\usepackage{multirow} 
\usepackage{enumitem}
\usepackage[T1]{fontenc}
\usepackage{listings}
\usepackage[most]{tcolorbox}
\usepackage{xcolor}
\usepackage[numbers,sort&compress]{natbib}

\usepackage{fvextra}
\usepackage{hyperref}

\usepackage{PRIMEarxiv}

\usepackage{amsmath}
\usepackage{amssymb}
\usepackage{mathtools}
\usepackage{amsthm}
\usepackage{graphicx}
\usepackage{subcaption}

\theoremstyle{plain}
\newtheorem{theorem}{Theorem}[section]
\newtheorem{proposition}[theorem]{Proposition}

\theoremstyle{definition}
\newtheorem{definition}[theorem]{Definition}
\newtheorem{assumption}[theorem]{Assumption}
\theoremstyle{remark}

\def \v x{\bm x}

\def \v x{\bm X}

\renewcommand{\v}[1]{\ensuremath{\boldsymbol{#1}}}

\usepackage[textsize=tiny]{todonotes}

\title{Characterizing Memorization in Diffusion Language\\ Models: Generalized Extraction and Sampling Effects}

\author{Xiaoyu Luo\textsuperscript{1},\space
  Wenrui Yu\textsuperscript{2},\space
  Qiongxiu Li\textsuperscript{2},\space
  Johannes Bjerva\textsuperscript{1}\\
  \textsuperscript{1}Department of Computer Science,\space
  \textsuperscript{2}Department of Electronic Systems\\
  Aalborg University, Copenhagen, Denmark\\
  \texttt{\{xilu,jbjerva\}@cs.aau.dk,\space \{wenyu,qili\}@es.aau.dk}
}

\begin{document}
\maketitle

\begin{abstract}

Autoregressive language models (ARMs) have been shown to memorize and occasionally reproduce training data verbatim, raising concerns about privacy and copyright liability. Diffusion language models (DLMs) have recently emerged as a competitive alternative, yet their memorization behavior remains largely unexplored due to fundamental differences in generation dynamics. To address this gap, we present a systematic theoretical and empirical characterization of memorization in DLMs.
We propose a generalized probabilistic extraction framework that unifies prefix-conditioned decoding and diffusion-based generation under arbitrary masking patterns and stochastic sampling trajectories. Theorem \ref{theo.sampling} establishes a monotonic relationship between sampling resolution and memorization: increasing resolution strictly increases the probability of exact training data extraction, implying that autoregressive decoding corresponds to a limiting case of diffusion-based generation by setting the sampling resolution maximal.
Extensive experiments across model scales and sampling strategies validate our theoretical predictions. Under aligned prefix-conditioned evaluations, we further demonstrate that DLMs exhibit substantially lower memorization-based leakage of personally identifiable information (PII) compared to ARMs.

\end{abstract}

\section{Introduction}
Large Language Models (LLMs) have achieved remarkable success and have been integrated into a wide range of everyday applications. Despite this progress, the field has long been dominated by the decoder-only architecture under the ARMs paradigm \cite{radford2018improving,radford2019language,brown2020language}. More recently, diffusion principles \cite{ho2020denoising} have been successfully incorporated into masked DLMs \cite{lou2023discrete,shi2024simplified,ou2024your}, leading to the emergence of Large DLMs \cite{nie2025scaling,nie2025large}. DLMs have demonstrated competitive scalability and performance with fewer training tokens. Since the training paradigm of these models fundamentally shifts from unidirectional next-token prediction to a bidirectional masking and reverse denoising process, their memorization behavior remains a critical yet unexplored frontier. Unlike ARMs that minimize the negative log-likelihood of sequential conditionals, DLMs optimizes a variational lower bound through a non-causal masking objective \cite{ou2024your}. This fundamental difference in how models internalize data necessitates a formal definition and systematic investigation of their memorization behavior, as it remains unclear how bidirectional denoising affects exposure to training data.

In the context of ARMs, a substantial body of work has shown that models can memorize and reproduce training samples verbatim \cite{carlini2021extracting,carlini2022quantifying,kiyomaru-etal-2024-comprehensive-analysis,huang2024demystifying,li2024rome,zhang2025extending,luo2025shared}. Accordingly, studying memorization is a practical necessity, as it is closely tied to privacy leakage \cite{huang2022large,kim2023propile} and copyright infringement \cite{cooper2025extracting,ahmed2026extracting}.

Standard discoverable extraction procedures typically rely on prefix–suffix prompting \cite{carlini2022quantifying,hayes2025measuring}. In contrast, DLMs differ in both conditioning and sampling, so we introduce an adapted evaluation framework to measure discoverable extraction under their bidirectional denoising paradigm.

In this work, we present the first systematic investigation into the memorization behavior of diffusion language models (DLMs). Our contributions are three-fold:
\begin{figure}[t]
    \centering
    \includegraphics[width=0.6\linewidth]{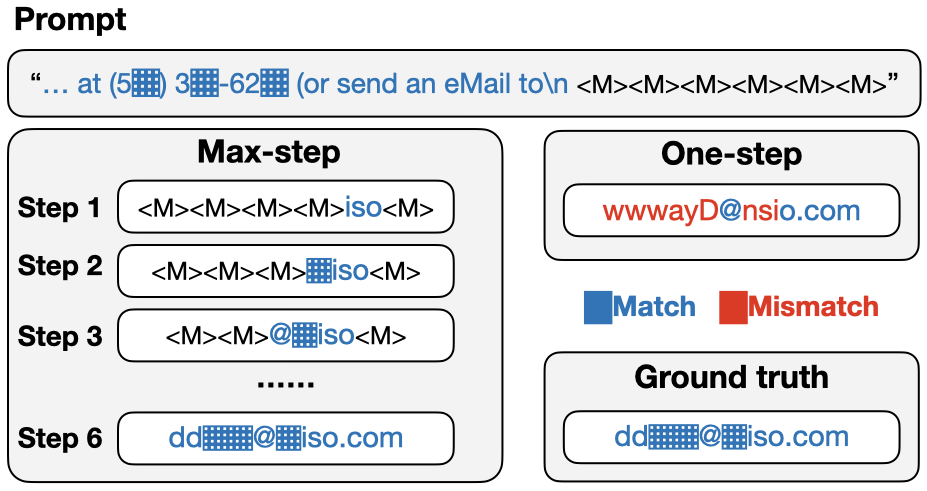}
    \caption{A randomly sampled verbatim PII memorization example from diffusion language model (\textsc{LLaDA-8B)}. This case shows that finer-grained generation resolution (i.e., more steps) is more prone to verbatim recovery of training data. $\langle$M$\rangle$ denotes a mask token; \textcolor{blue}{blue} highlights matched strings, and \textcolor{red}{red} highlights mismatches.}
    \label{fig:firstpage}
\end{figure}
\begin{itemize}
    \item \textbf{A generalized formulation of memorization for diffusion language models.}
    We develop a generalized probabilistic formulation of discoverable memorization that applies to DLMs with arbitrary masking patterns and stochastic sampling trajectories, under which the prefix-conditioned $(n,p)$-discoverable extraction for ARMs arises as a special case.
 
    \item \textbf{Sampling resolution provably controls memorization.}
    We prove a monotonicity result showing that, under a natural recovery assumption, increasing the sampling resolution (i.e., recovering masked tokens in finer-grained denoising steps) monotonically increases the probability of exact memorization, and we demonstrate this effect empirically across sampling strategies and model scales.
 
    \item \textbf{Aligned analysis of PII memorization across model scales and architectures.}
    We align autoregressive and diffusion language models under the same prefix-conditioned PII completion setting and quantify memorization across multiple model scales and architectures, finding that autoregressive models exhibit higher PII memorization and leakage risk.
\end{itemize}

\section{Related Work}
\subsection{Language Model Architecture}
\paragraph{The Autoregressive Paradigm.} The contemporary landscape of LLMs is predominantly defined by the ARMs paradigm, which factorizes the joint probability of sequences into a product of sequential conditionals \cite{radford2018improving,radford2019language,brown2020language}. Leveraging the Transformer architecture, ARMs have demonstrated exceptional scalability and the emergence of complex capabilities \cite{wei2022emergent} under empirical scaling laws \cite{kaplan2020scaling,hoffmann2022training}. Despite their success, ARMs exhibit inherent limitations directly tied to their unidirectional nature, most notably the "reversal curse"—a systemic failure to generalize bidirectional relationships \cite{berglund2023reversal}.

\paragraph{Masked Diffusion Paradigm.} Diffusion models have seen success in the computer vision domain, typically formulated as a continuous generative task \cite{ho2020denoising,song2020score}. For the text-based task, one line of research explores continuous diffusion by embedding tokens into continuous spaces \cite{li2022diffusion,gong2022diffuseq,zhang2023planner}. In another discrete line, masked diffusion paradigm has emerged as a principled generative alternative for language modeling, leveraging a unique probabilistic framework that enables flexible bidirectional context modeling \cite{austin2021structured,sahoo2024simple,shi2024simplified}. This architecture is characterized by a forward process that either preserves tokens or replaces them with a unique mask token, and is trained with a reconstruction objective that approximates the reverse generative process. Foundational theoretical frameworks for Masked DLMs have established their connection to maximum likelihood estimation and variational lower bounds \cite{ou2024your}. Empirically, Masked DLMs have demonstrated promising results in smaller-scale settings \cite{lou2023discrete}.

\paragraph{Large Diffusion Language Models.} Large DLMs are rapidly advancing diffusion-based language modeling into a regime that is competitive with mainstream ARMs at a comparable scale. \citet{nie2025scaling} introduces the first formal scaling law for Masked DLMs, showing an ARM-like power-law improvement with scale. Building on this line of work, LLaDA \cite{nie2025large} follows the standard modern LLM training recipe with large-scale pretraining followed by supervised fine-tuning (SFT), and scales large DLMs to 8B parameters. It delivers broadly competitive performance against the same scale ARMs baseline on a suite of zero-/few-shot benchmarks, and shows in-context learning comparable to strong 8B-class models such as LLaMA3 \cite{llama3modelcard}. Beyond accuracy, large DLMs’ bidirectional conditioning improves robustness to directionality, outperforming GPT-4o \cite{openai_gpt4o_2024} in the reversal settings completion task. Most recently, LLaDA2.0 \cite{bie2025llada2} further scales diffusion language models to 16B/100B and shows that performance continues to improve at scale, achieving strong results on broad evaluations, while retaining the efficiency benefits of parallel generation.

\subsection{Language Model Memorization}
Prior work has established that LLMs can memorize and regurgitate portions of their training data \cite{carlini2021extracting,carlini2022quantifying,huang2024demystifying,li2024rome,zhang2025extending}. Early evidence shows that black-box adversaries can extract hundreds of verbatim training sequences, including rare and sensitive content, from generative models~\cite{carlini2021extracting}.  Subsequent work quantifies how extractable memorization, showing consistent increases with both model capacity and training-data repetition~\cite {carlini2022quantifying}.

Beyond a strong verbatim-controlled setting, a growing body of research investigates memorization risks in realistic scenarios involving personally identifiable information (PII). Studies on PII leakage indicate that LLMs may retain sensitive personal information, which can be reconstructed through association-based attacks~\cite{huang2022large,kim2023propile}.  From a complementary angle, ~\citet{lukas2023analyzing} formalize and evaluate PII leakage under black-box access, showing that semantic associations can enable adversaries to recover private attributes via indirect prompting. 
More recently,~\citet{luo2026llms} argue that memorization estimates may be inflated when prompts contain implicit cues, leading to an overestimation of a model’s true memorization behavior.

Memorization and data extraction have also been examined in the context of production and aligned language models.~\citet{nasr2025scalable} show that alignment safeguards can be circumvented at scale, enabling the extraction of memorized training samples from deployed proprietary systems.  In parallel, analyses of open-weight models reveal that certain copyrighted books are memorized to a much greater extent than previously understood, in some cases allowing near-complete reconstruction via probabilistic extraction~\cite{cooper2025extracting}. 
Extending these observations to deployed systems,~\citet{ahmed2026extracting} demonstrates that, under specific prompting and jailbreak conditions, substantial portions of in-copyright books can be extracted near-verbatim, highlighting persistent memorization risks despite system-level safeguards.

Most prior evaluations of memorization rely on greedy or single-query extraction settings, which may not reflect realistic deployment conditions. To address this limitation, ~\citet{hayes2025measuring} introduces probabilistic extractable memorization metrics that measure the expected query effort required to recover a memorized sequence under non-deterministic decoding and repeated querying, providing a more deployment-relevant assessment of memorization risk.

Finally, memorization has also been studied in multilingual settings. 
~\citet{luo2025shared} extend memorization measurements to multilingual LLMs and identify a long-tail \cite{li2025trustworthy} effect mediated by language similarity: within similar language groups, lower-resource languages tend to exhibit higher levels of memorization.

\section{Preliminaries}

In this section, we briefly review DLMs and ARMs architectures, introduce their commonly used sampling paradigms, and finally present the existing $(n,p)$-discoverable extraction framework defined for ARMs.

\subsection{Masked Diffusion Model}
We review the framework of the masked diffusion paradigm \cite{ou2024your,nie2025large}, which serves as the fundamental architecture for Large DLMs.

\textbf{Forward Process.}
Let $\mathcal{V} = \{0, 1, \dots, K-1\}$ be the vocabulary set. Given a sequence $\v{z}_0 \in \mathcal{V}^L$ of length $L$, the forward process at noise level $t \in [0, 1]$ gradually replaces tokens. The transition probability is defined as
\begin{equation}
    q_{t|0}(\v{z}_t | \v{z}_0) = \prod_{l=1}^{L} q_{t|0}( z_t^l | z_0^l),
\end{equation}
where $q_{t|0}(z_t^l | z_0^l) = \begin{cases} \alpha_t, & z_t^l = z_0^l \\ 1 - \alpha_t, & z_t^l = m \end{cases}$. $\v z_t$ denotes the noisy data at time $t$, $z_t^l$ denote the $l$-th element of $\v z_t$, $m$ denotes a special [MASK] token \cite{devlin2019bert} and $\alpha_t = 1 - t$.

\textbf{Reverse Process.}
The reverse process will recover the original tokens $\v{z}_0$ from a fully masked sequence $\v{z}_1$. For any $0 \leq s < t \leq 1$, the transition $q_{s|t}(\v{z}_s | \v{z}_t)$ recovers a portion of the masked tokens, which can be formulated as
\begin{equation}
    q_{s|t}(z_s^l | \v{z}_t) = \begin{cases} 
    1, & z_t^l \neq m, z_s^l = z_t^l \\ 
    \frac{s}{t}, & z_t^l = m, z_s^l = m \\ 
    \frac{t-s}{t} q_{0|t}(z_s^l | \v{z}_t), & z_t^l = m, z_s^l \neq m \\ 
    0, & \text{otherwise} \end{cases}
\end{equation}

\textbf{Training Objective.}
A model $p_\theta(z_0^l | \v{z}_t)$ is trained to approximate the true data distribution by minimizing the following negative log-likelihood upper bound as
\begin{equation}
    \mathcal{L} = \int_0^1 \frac{1}{t} \mathbb{E}_{q(\v{z}_t|\v{z}_0)} \left[ \sum_{l: z_t^l = m} -\log p_\theta(z_0^l | \v{z}_t) \right] dt.
\end{equation}

\textbf{Sampling Strategy.}
Sampling iteratively updates the currently masked positions while keeping unmasked tokens unchanged. A \emph{linearly spaced} time grid $\{t_i\}_{i=0}^{N}$ is used, with $t_0=1$ and $t_{i+1}<t_i$, $N$ is the sampling resolution. The linear schedule determines the fraction of tokens to be recovered at each step $t\!\to\! s$
by $1-\frac{s}{t}$. At step $t\!\to\! s$, for each masked position $l$ ($z_t^l=m$), the model predicts a token
$\hat z_0^l=\arg\max_{v\in\mathcal V} p_\theta(v\mid\v z_t)$ with confidence
$p_\theta(\hat z_0^l\mid\v z_t)$.
\emph{Greedy sampling} selects the masked positions with the highest confidence and recovers
$k=\lfloor |\mathcal M|(1-\frac{s}{t})\rfloor$ tokens, while \emph{random sampling} recovers the
same number of tokens by uniformly sampling positions from $\mathcal M$.

\subsection{Autoregressive Model}
Most existing LLMs follow the ARM formulation \cite{radford2018improving}. Given a sequence $\v{z}\in\mathcal{V}^L$, an ARMs factorizes
\begin{equation}
    p_\theta(\v{z}) = \prod_{l=1}^{L} p_\theta\!\left(z^l \mid \v z^{1:l-1}\right),
\end{equation}
where $\v z^{1:l-1}=[z^1,\dots,z^{l-1}]$ is the prefix; we use this as the canonical baseline paradigm throughout.

\subsection{Sampling Paradigm}
We then briefly introduce the sampling paradigms that are commonly used in ARMs and DLMs. Among these paradigms, temperature and top-$k$ sampling are widely adopted in ARMs \cite{fan2018hierarchical}, while DLMs use Gumbel noise perturbation \cite{nie2025scaling}.

\paragraph{Temperature Sampling.}
 Given a prefix $\v z^{1:l-1}$, let $\ell_v$ denote the logit for token $v \in \mathcal{V}$. Temperature sampling parameterizes the predictive categorical distribution as
\[
\text{Pr}(z^l=v \mid \v z^{1:l-1}) = \frac{\exp(\ell_v/T)}{\sum_{u\in \mathcal{V}}\exp(\ell_u/T)}, \quad T>0.
\]
Here, $T$ serves as a scaling factor that controls the sharpness of the distribution: a lower $T$ concentrates the probability mass on high-logit tokens, thereby reducing stochasticity, while a higher $T$ flattens the distribution, promoting diversity in generation. The standard softmax distribution is recovered as a special case when $T=1$.

\paragraph{Top-$k$ sampling.}
Let $p(v)=\mathrm{softmax}(\ell_v)$ and let $\mathcal{K}$ be the set of $k$ tokens with largest $p(v)$. Top-$k$ sampling restricts the distribution to $\mathcal{K}$ and renormalizes it as
\[
Pr(z^l=v \mid \v z^{1:l-1}) \;=\;
\begin{cases}
\frac{p(v)}{\sum_{u\in \mathcal{K}} p(u)}, & v \in \mathcal{K},\\
0, & \text{otherwise}.
\end{cases}
\]
This procedure collapses to deterministic greedy decoding when $k=1$, where only the most probable token is selected.

\paragraph{Gumbel noise perturbation.}
In diffusion models, a commonly used sampling paradigm is Gumbel noise perturbation.
To perform stochastic token sampling, we draw i.i.d. random variables
$U_v \sim \mathrm{Uniform}(0,1)$ for each vocabulary item $v \in \mathcal{V}$,
and construct the perturbed scores as
\begin{equation}
S_v \;=\; \frac{\exp(\ell_v)}{\big(-\log U_v\big)^{T_{\mathrm{Gumbel}}}},
\end{equation}
where $T_{\mathrm{Gumbel}}$ is a temperature parameter that controls the level of stochasticity.
The sampled token is then selected as the most confident one, i.e.,
\[
\arg\max_{v \in \mathcal{V}} S_v.
\]
When $T_{\mathrm{Gumbel}} = 0$, $S_v = \exp(\ell_v)$.
In this case, the procedure reduces to greedy decoding, which can be naturally paired with
deterministic sampling in diffusion models.

\subsection{$(n,p)$-discoverable Extraction Framework}\label{ssec.hayes}

We review two definitions of memorization for ARMs as introduced by \citet{hayes2025measuring}.

Let $[L] = \{1,2, \dots, L\}$ be the set of all token indices and $\v{z} = [z_1, \dots, z_L]\in\mathcal{V}^L$ be a training example. 
\begin{definition}[$(n,p)$-discoverable Extraction]
A sequence is $(n,p)$-discoverably extractable if, given a prefix, the model generates the exact suffix within $n$ independent queries with probability at least $p$.
\end{definition}

Let $p_{\v{z}}$ denote the probability of generating the exact suffix in a single trial. Consequently, the probability of failing to recover the correct suffix in $n$ independent trials is $(1 - p_{\v{z}})^n$. In order to ensure that the probability of at least one successful extraction is at least $p$, we require $1 - (1 - p_{\v{z}})^n \geq p$. Solving for $n$, we obtain the estimation for the required number of queries as
\begin{equation}
     n \geq \frac{\log (1-p)}{\log (1-p_{\v{z}})}.
\end{equation}

This relationship implies that the extractability parameters $(n, p)$ can be determined once $p_{\v{z}}$ is known. In practice, $p_{\v{z}}$ is computed as the cumulative product of the conditional probabilities associated with recovering each token at its corresponding sampling step.

Beside the exact extraction, \citet{hayes2025measuring} also proposed a relaxed version for approximate recovery.
\begin{definition}[$(\epsilon, n, p)$-discoverable Extraction]
A suffix of an example is $(\epsilon, n, p)$-discoverably extractable if the probability that at least one of $n$ independent queries results in a generated suffix with a distance $\le \epsilon$ from the ground truth is at least $p$.
\end{definition}

\section{A Generalized Discoverable Memorization Framework for Diffusion Language Models}

Existing definitions of discoverable memorization~\citep{hayes2025measuring} are inherently tied to the autoregressive paradigm, where memorization is evaluated through a fixed prefix--suffix structure under left-to-right decoding. While this formulation closely mirrors the unidirectional generation process of autoregressive language models, it does not naturally extend to diffusion language models, which recover masked tokens through stochastic, non-sequential denoising trajectories.

To address this mismatch, we develop a generalized formulation of discoverable memorization that removes the prefix-based decoding assumption and explicitly accounts for arbitrary masking patterns and diffusion sampling trajectories. 
 Moreover, within our proposed framework, prefix-conditioned extraction can be recovered as a special case (see Proposition~\ref{prop.special}), allowing memorization in diffusion language models to be analyzed under their native generation dynamics. 


\subsection{Generalized Exact Memorization in DLMs}

We first focus on the case of exact memorization.

\subsubsection{Definition and Approximation}
\begin{definition}[Generalized $(n,p)$-discoverable Extraction]\label{dif.ourmem}
A sequence $\v{z}$ is $(n,p)$-discoverably extractable under a mask $\mathcal{M} \subset \{1, \dots, L\}$ if, given the observed tokens $\v{z}_{\bar{\mathcal{M}}}$, the model recovers the exact original tokens at all masked positions $\mathcal{M}$ within $n$ independent queries with probability at least $p$. 

Formally, let $\hat{\v{z}}^{(i)}_{\mathcal{M}}$ be the output of the $i$-th independent sampling. The sequence $\v{z}$ is $(n,p)$-extractable if:
\begin{equation}
    \text{Pr}\left( \exists i \in \{1, \dots, n\} : \hat{\v{z}}^{(i)}_{\mathcal{M}} = \v{z}_{\mathcal{M}} \mid \v{z}_{\bar{\mathcal{M}}} \right) \geq p,
\end{equation}
where $\bar{\mathcal{M}}=[L]\backslash\mathcal{M}$ denotes the complement set of indices (the observed positions).
\end{definition}

Following \citet{hayes2025measuring}, the required number of independent queries $n$ can be estimated given a target extraction probability $p$ and an approximation of the single-trial recovery probability $p_{\v z}$. In the context of diffusion-based LLMs, the approximation of $p_{\v z}$ is a bit different and it is inherently dependent on the sampling resolution $N$. 

For a sampling process with $N$ steps in the $i$-th independent query, we partition $\mathcal{M}$ into $N$ disjoint chunks $\{\Delta_1^{(i)}, \Delta_2^{(i)}, \dots, \Delta_N^{(i)}\}$ such that $\Delta_k^{(i)}$ represents the set of token indices recovered at step $k$. Let $\mathcal{U}_k^{(i)}$ be the set of observed indices after step $k$, with $\mathcal{U}_0^{(i)} = [L]\backslash\mathcal{M}$ and $\mathcal{U}_k^{(i)} = \mathcal{U}_{k-1}^{(i)} \cup \Delta_k^{(i)}$. 
Because tokens are recovered independently at each step (given the unmasked context), the total probability $p_{\v z}^{(i)}$ is defined by the multiplicative product of their respective conditional probabilities
\begin{align}
\label{eq.pz}
    \nonumber p_{\v z}^{(i)} &= \text{Pr}(\hat{\v z}_{\mathcal{M}} = \v z_{\mathcal{M}} \mid \v z_{\bar{\mathcal{M}}}, N) \\
    &= \prod_{k=1}^N \prod_{\pi \in \Delta_k^{(i)}} \text{Pr}(\hat{z}_{\pi} = z_{\pi} \mid \v z_{\mathcal{U}_{k-1}^{(i)}}).
\end{align}

Unlike ARMs, where the decoding order is typically fixed, diffusion-based models exhibit stochasticity in the recovery path (i.e., the sequence of token subsets selected at each step). Consequently, the single-trial probability $p_{\v z}$ cannot be reliably estimated from a single query, as is possible in the deterministic setting of autoregressive. To account for this randomness, we estimate the expected $\hat{p}_{\v z}$ by averaging results over a set of $R$ trials, i.e. $\hat{p}_{\v z}=\frac{1}{R}\sum_{i=1}^{R}p_{\v z}^{(i)}$. The validation results are in Sec~\ref{sec:emp_difi}.

\subsubsection{The Impact of Sampling Resolution}

\begin{assumption}[Monotonicity of Recovery Probability]\label{ass.mono}
We assume that for any inference step, the probability of correctly recovering a subset of masked tokens increases as the set of observed tokens expands. 

Formally, for any two observation index sets $\mathcal{U}, \mathcal{W} \subseteq [L]$ such that $\mathcal{U} \subset \mathcal{W}$, and for any target mask set $\mathcal{M} \subseteq [L] \setminus \mathcal{W}$, we have
\begin{equation}
    \text{Pr}(\hat{\v{z}}_{\mathcal{M}} = \v{z}_{\mathcal{M}} \mid \v{z}_{\mathcal{W}}) \geq \text{Pr}(\hat{\v{z}}_{\mathcal{M}} = \v{z}_{\mathcal{M}} \mid \v{z}_{\mathcal{U}})
\end{equation}
where $\v{z}_{\mathcal{U}}$ and $\v{z}_{\mathcal{W}}$ denotes the sub-vector of tokens at indices in $\mathcal{U}$ and $\mathcal{W}$ respectively, and $\hat{\v{z}}_{\mathcal{M}}$ is the reconstruction of tokens at indices $\mathcal{M}$. 
\end{assumption}

This assumption formalizes the intuitive property that additional \textit{correct context} should not reduce recovery likelihood, which is also validated empirically in Sec~\ref{sec:6.2}.
It is important to note that we specifically consider context consisting of correctly recovered tokens, rather than potentially erroneous predictions. This focus is consistent with our objective of analyzing \textit{exact memorization}, where the recovery of the entire sequence is only successful if every intermediate step along the sampling trajectory remains faithful to the training instance. 

Hereby we have the following Theorem.
\begin{theorem}[Impact of Sampling Resolution on Memorization]\label{theo.sampling}
Under Assumption~\ref{ass.mono}, for diffusion-based LLMs, the probability of generating a sequence that exactly matches a training instance generally increases with the number of sampling steps $N$.

In certain cases, this relationship is theoretically monotonic under a fixed recovery sequence. Let $\sigma = (\pi_1, \pi_2, \dots, \pi_{M})$ be a fixed permutation of indices in $\mathcal{M}$ representing the recovery order. For any $N$, we partition $\sigma$ into $N$ disjoint and ordered chunks $\mathcal{P}_N = \{\Delta_1, \Delta_2, \dots, \Delta_N\}$. 
If $N_1 > N_2$ and the partition $\mathcal{P}_{N_1}$ is a refinement of $\mathcal{P}_{N_2}$ (i.e., each chunk in $\mathcal{P}_{N_2}$ is the union of one or more consecutive chunks in $\mathcal{P}_{N_1}$), then:
\begin{equation}
    \text{Pr}(\hat{\v{z}}_{\mathcal{M}} = \v{z}_{\mathcal{M}} \mid \v{z}_{\bar{\mathcal{M}}}, N_1) \geq \text{Pr}(\hat{\v{z}}_{\mathcal{M}} = \v{z}_{\mathcal{M}} \mid \v{z}_{\bar{\mathcal{M}}}, N_2).
\end{equation}
\end{theorem}

\begin{proof}
    See Appendix~\ref{app.theo1proof}.
\end{proof}

This theorem establishes that increasing the sampling resolution consistently enhances the probability of exact memorization. Since the number of sampling steps $N$ is upper-bounded by the total number of masked tokens $|\mathcal{M}|$, the most granular refinement occurs when $N=|\mathcal{M}|$. In this limiting case, the diffusion model recovers tokens in one-by-one fashion, mirroring the fundamental mechanism of ARMs. This leads to the following proposition.
\begin{proposition} \label{prop.special}
Conceptually, ARMs can be viewed as a special case of a diffusion-based LLM where the sampling resolution $N$ reaches its maximum value ($N=|\mathcal{M}|$). In this regime, the model recovers only one token per step in a fixed sequential order. Following the logic of Theorem~\ref{theo.sampling}, as the sampling process transitions from parallel refinement to fully sequential generation, the increasing availability of step-wise context suggests that the ARM-style decoding represents an upper-limit behavior for exact memorization probability.
\end{proposition}

\subsection{Generalized Relaxed Memorization in DLMs}

\begin{definition}[Generalized $(\epsilon,n,p)$-discoverable Extraction]
We generalize the notion of discoverable extraction to arbitrary masking patterns. An example $\v{z}$ is $(\epsilon, n, p)$-discoverably extractable under a mask $\mathcal{M}$ if, given the observed tokens $\v{z}_{\bar{\mathcal{M}}}$, the model generates a reconstructed sequence $\hat{\v{z}}_{\mathcal{M}}$ such that the distance to the original tokens $d(\hat{\v{z}}_{\mathcal{M}}, \v{z}_{\mathcal{M}}) \leq \epsilon$ holds for at least one of $n$ independent queries with probability at least $p$, i.e.,
\begin{equation}
    \text{Pr}\left( \exists i \in \{1, \dots, n\} : d(\hat{\v{z}}^{(i)}_{\mathcal{M}}, \v{z}_{\mathcal{M}}) \leq \epsilon \mid \v{z}_{\bar{\mathcal{M}}} \right) \geq p,
\end{equation}
where $d(\cdot, \cdot)$ is a distance metric. When $\epsilon = 0$, this definition reduces to the case of exact extraction.
\end{definition}

In \citet{hayes2025measuring}, the error-tolerant extraction probability $p_{\v z, \epsilon}$ can be empirically estimated by calculating the proportion of successful reconstructions (the distance is within the threshold $\epsilon$) across $R$ independent trials, i.e.,
\begin{equation}\label{eq.enp_estimation}
    \hat{p}_{\v z, \epsilon} = \frac{1}{R} \sum_{i=1}^R \mathbb{I}\left( d(\hat{\v z}^{(i)}_{\mathcal{M}}, \v z_{\mathcal{M}}) \le \epsilon \right).
\end{equation}
 This estimation method naturally extends to the generalized definition. Specifically, when the error tolerance is set to $\epsilon = 0$, the estimator $\hat{p}_{\v z, 0}$ provides a direct empirical approximation for the exact memorization probability $\text{Pr}(\hat{\v z}_{\mathcal{M}} = \v z_{\mathcal{M}} \mid \v z_{\bar{\mathcal{M}}}, N)$ with a large $R$.

In addition, we investigate the statistical behavior of the error distribution under large-scale trials. We refer the reader to Appendix~\ref{app.enp} for the related experimental analysis.

\section{Experimental Setup}
\label{sec:setup}
\subsection{Model}
\paragraph{Foundation model (DLM and ARM).}
To isolate scale and architectural effects on memorization, we pretrain our DLMs and an ARM baseline under an identical setup. Specifically, we reproduce the training recipe of \citet{nie2025scaling}.\ under a fixed compute budget of $10^{21}$ FLOPs, and train three DLMs with 170M, 690M, and 1.1B parameters on SlimPajama. For a one-to-one paradigm comparison, we also train a 1.1B ARM baseline with the \emph{exact same} setup (and tokenizer) as the 1.1B DLM. It ensures that any differences in memorization behavior are attributable primarily to the scale and generative architecture, rather than to data, tokenization, or training procedure.

\paragraph{Finetuning (LLaDA-8B alignment).}
Since LLaDA-8B’s original pretraining data are not publicly available, we evaluate memorization after finetuning on the Enron email dataset. To ensure comparability across model families, we use the same finetuning protocol for LLaDA-8B and all SlimPajama-pretrained DLMs and ARMs, continuing training for one epoch on Enron with each model’s end-of-pretraining learning rate. This shared adaptation setting ensures memorization differences are not driven by mismatched finetuning procedures.

\subsection{Dataset Collection}
\paragraph{Validating Data Collection.}
For the validation experiments (Sec \ref{sec:emp_difi}, \ref{sec:6.2}), we draw evaluation samples from the pretrain dataset by uniformly sampling \textbf{100-token windows} at random from the corpus. For each window, we apply random masking with mask ratio $r \in \{0.20, 0.25, 0.30\}$, and evaluate reconstruction/extraction under these different masking levels.

\paragraph{PII Collection.}
For the alignment experiment (Sec. \ref{sec:align}), we construct a PII-focused evaluation set from the Enron email dataset, extracting emails and phone numbers via regular expressions, following \cite{kim2023propile} procedure (see Appendix~\ref{app:regex} for details). We sample \textbf{3,000} email-containing texts and \textbf{3,000} phone-number-containing texts, and for each PII occurrence use the 100 preceding tokens as the prefix to test \textbf{verbatim memorization} under prefix-conditioned generation. To avoid spurious pattern-based cues, email prompts are further filtered to contain no additional email addresses. 

\section{Analysis and Results}
This section empirically evaluates our memorization framework. We first validate that generalized discoverable extraction is measurable in DLMs despite stochastic masking and iterative denoising, and then confirm that finer sampling resolution increases verbatim memorization. We further align DLMs and ARMs in a PII completion setting to compare memorization-based leakage under identical prefix-conditioned reconstruction, and finally show the metric reflects memorization rather than generalization via train vs.\ disjoint same-domain test comparisons.
\subsection{Rationality of generalized discoverable extraction in DLMs}\label{sec:emp_difi}
\begin{figure}[ht]
    \centering
    \includegraphics[width=0.6\linewidth]{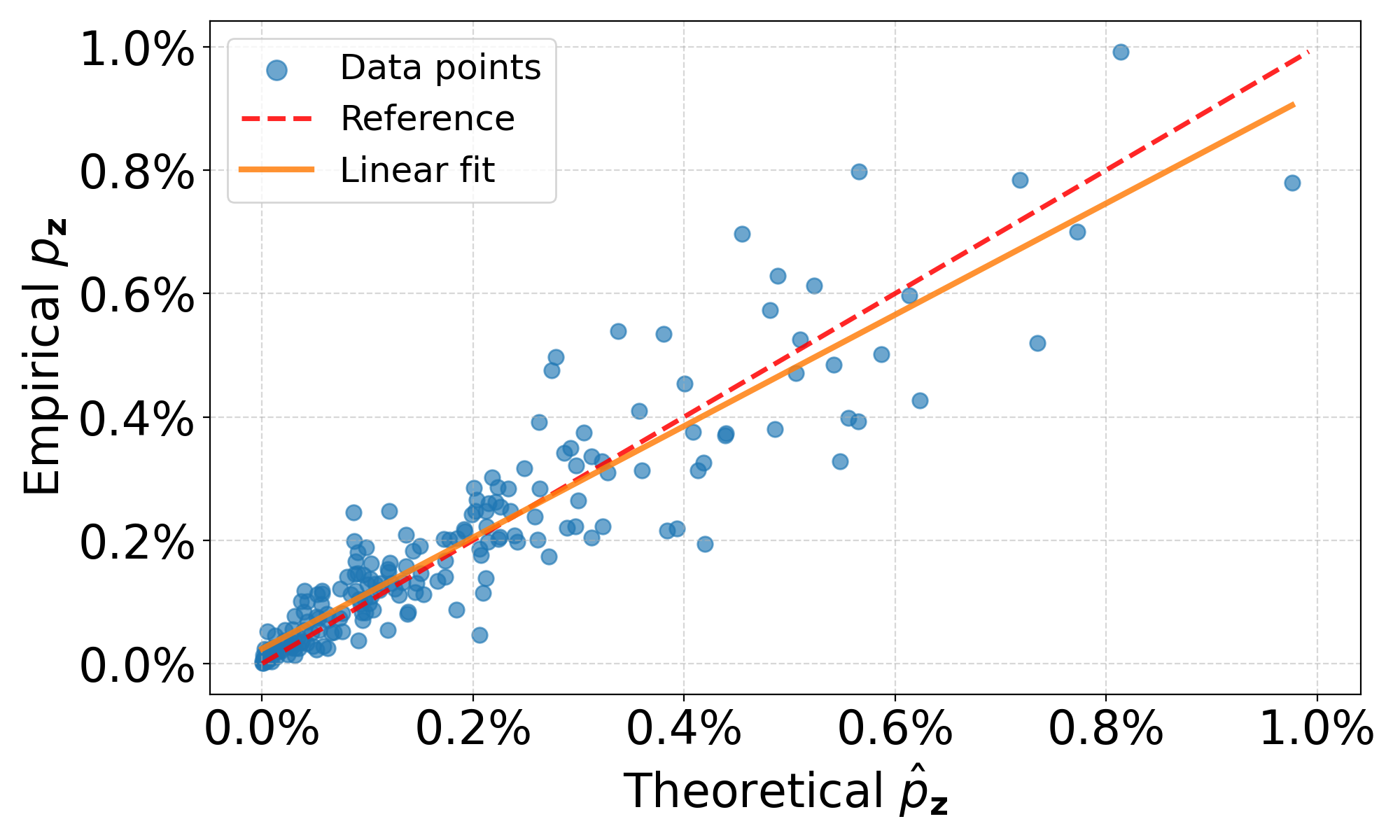}
    \caption{Empirical vs.\ theoretical memorization probability measured by \textsc{DLM-1.1B} on 200 SlimPajama examples. Empirical $p_{\v z}$ is estimated from 100{,}000 random-decoding generations; theoretical $\hat{p}_{\v z}$ is estimated from 1{,}024 mask-patterns at a fixed mask ratio ($[0.2,0.3]$).}

    \label{fig:pp}
\end{figure}
To validate that our proposed notion of memorization is operational and measurable in practice (Definition~\ref{dif.ourmem}), we quantify the memorization probability at the level of individual examples. 

Since DLMs decoding depends on stochastic masking and iterative denoising, we fix the mask ratio and sample multiple mask-patterns (i.e., different masked positions); we then aggregate the one-step recovery outcomes across these mask-patterns to estimate the probability of recovering the target content.

Figure~\ref{fig:pp} compares the theoretical $\hat{p}_{\v z}$ with the empirical $p_{\v z}$. For computational feasibility, we evaluate on $200$ examples, uniformly sampled from those with a preliminary estimated recovery probability above $10^{-3}$ (screened using $1{,}024$ mask-patterns at a fixed mask ratio in $[0.2,0.3]$). For these selected examples, we then re-estimate $\hat{p}_{\v z}$ from scratch using a fresh, independent set of $1{,}024$ mask-patterns, and compute $p_{\v z}$ from $100{,}000$ randomly decoded generations. The two distributions exhibit a high level of agreement throughout the range, suggesting that \eqref{eq.pz} reliably approximates the underlying recovery probability with relatively few mask-patterns. Additional results under different numbers of mask-patterns are provided in Appendix~\ref{appendix:sampling-resolution}.

\subsection{Impact of Sampling Resolution on Verbatim Memorization}\label{sec:6.2}
\begin{figure}[ht]
    \centering
    \includegraphics[width=0.5\linewidth]{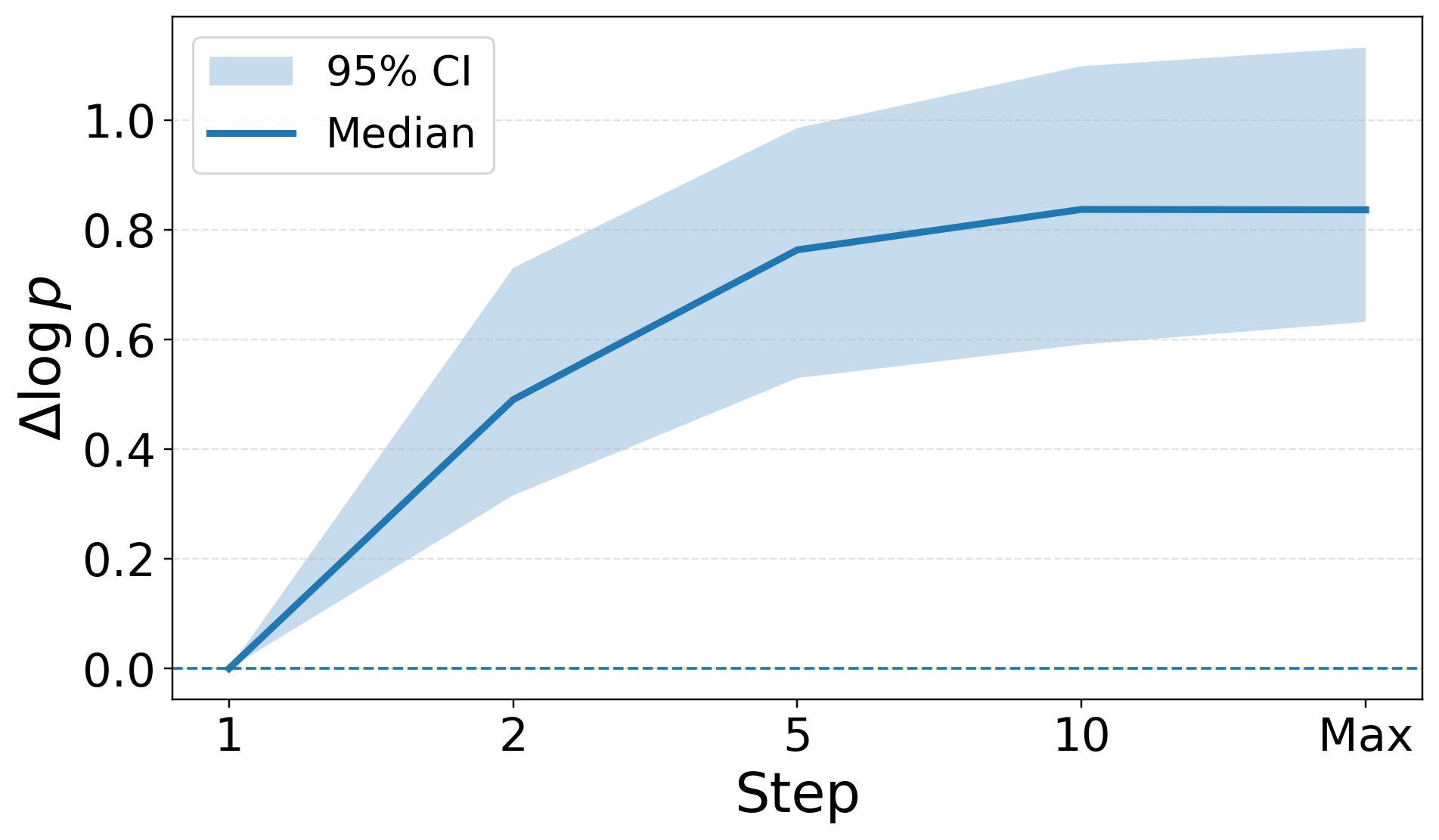}
    \caption{Empirical log-probability ratio across decoding steps (\textsc{DLM-1.1B}). We sample 200 mask-patterns and, for each, run $10{,}000$ random decoding generations to estimate success probabilities at steps $\{1,2,5,10,\text{per-token}\}$. We plot the median log-probability difference to one-step (baseline), with $95\%$ bootstrap Confidence Intervals (CIs) across generation-trajectories.}

    \label{fig:lnpgainmedian}
\end{figure}
In this section, we provide empirical evidence supporting Assumption~\ref{ass.mono}. Building on Sec~\ref{sec:emp_difi}, we select 200 mask-patterns from the previous section whose theoretical one-step recovery rate exceeds $0.01$, a subset chosen for computational feasibility. For each selected trajectory, we run \textbf{$10{,}000$} unconstrained generation trials under five generation resolutions: $1$-step ($N=1$), $2$-step ($N=2$), $5$-step ($N=5$), $10$-step ($N=10$), and Max ($N=|\mathcal{M}|$, where the number of masked tokens is used as the generation step.)

Figure~\ref{fig:lnpgainmedian} summarizes these \emph{empirical} exact-recovery success rates by plotting, for each resolution, the cross-trajectory median of the log success-rate ratio relative to the one-step baseline. We report $\Delta \log p := \mathrm{median}\!\bigl(\log(\hat{p}_N/\hat{p}_1)\bigr)$, where $\hat{p}_N$ is the empirical exact-recovery hit rate at resolution $N$ and the median is taken over the mask-patterns samples.
The resulting curve exhibits a clear upward trend as the number of generation steps increases, suggesting an approximately monotonic relationship between generation resolution and exact-recovery success. This implies that allocating more denoising steps consistently increases the likelihood of reproducing the target content exactly, providing empirical support for Assumption~\ref{ass.mono}.  Additional memorization examples are provided in Appendix~\ref{app:memexp}.

\subsection{Empirical Analysis of PII Leakage via Aligned ARM–LDM Measurement}\label{sec:align}
\begin{table}[t]
\centering
\small
\setlength{\tabcolsep}{3pt}
\begin{tabular}{llcccc}
\toprule
\multirow{2}{*}{\textbf{Model}} & \multirow{2}{*}{\textbf{Step}} &
\multicolumn{2}{c}{\textbf{Email}} &
\multicolumn{2}{c}{\textbf{Phone}} \\
\cmidrule(lr){3-4}\cmidrule(lr){5-6}
& & $p$ = 50\% & $p$ = 99\% & $p$ = 50\% & $p$ = 99\% \\
\midrule

\multirow{2}{*}{\textsc{DLM-1.1B}}
  & One & 0   & 0   & 0 & 0 \\
  & Max & 16  & 7   & 0 & 0 \\
\cmidrule(lr){1-6}

\textsc{ARM-1.1B}
  & N/A & 213 & 139 & 5 & 3 \\
\midrule

\multirow{2}{*}{\textsc{LLaDA-8B}}
  & One & 9   & 8   & 7  & 3 \\
  & Max & 179 & 122 & 23 & 17 \\
\bottomrule
\end{tabular}
\caption{Generalized $(n,p)$-discoverable Extraction. Number of memorized samples at target probabilities $p$ under a query budget of $n=10{,}000$ on $3{,}000$ completion prompts, for different PII types.}
\label{tab:nptable}
\end{table}

Under the aligned PII prefix-suffix verbatim completion task, DLMs exhibit substantially lower $(n,p)$-discoverable PII memorization than scale-matched ARMs. Table~\ref{tab:nptable} reports the number of memorized email and phone samples at target probabilities $p\in\{0.5,0.99\}$ under a query budget of $n=10{,}000$ over $3{,}000$ prompts. This gap may partly reflect differences between training objectives: random masking in DLMs disrupts learning to preserve long contiguous context, while the ARMs training paradigm naturally matches the sequential prefix-suffix decoding. 

At a larger scale, the \textsc{LLaDA-8B} model exhibits a clear increase in memorization when moving from single-step reconstruction to per-token reconstruction, providing practical task evidence for Assumption~\ref{ass.mono} that finer-grained denoising increases exact recovery probability. However, even under per-token reconstruction, the 8B diffusion model is only comparable to the 1.1B ARM. Importantly, this comparison should be interpreted in light of training cost: LLaDA-8B is trained with roughly two orders of magnitude more FLOPs and about $150\times$ more training token exposure than the 1.1B ARM baseline (15.6B vs.\ 2300B) \cite{nie2025large}. Overall, in this aligned prefix-conditioned PII setting, DLMs exhibit lower memorization behavior.

\begin{figure}[h]
    \centering
    \includegraphics[width=0.5\linewidth]{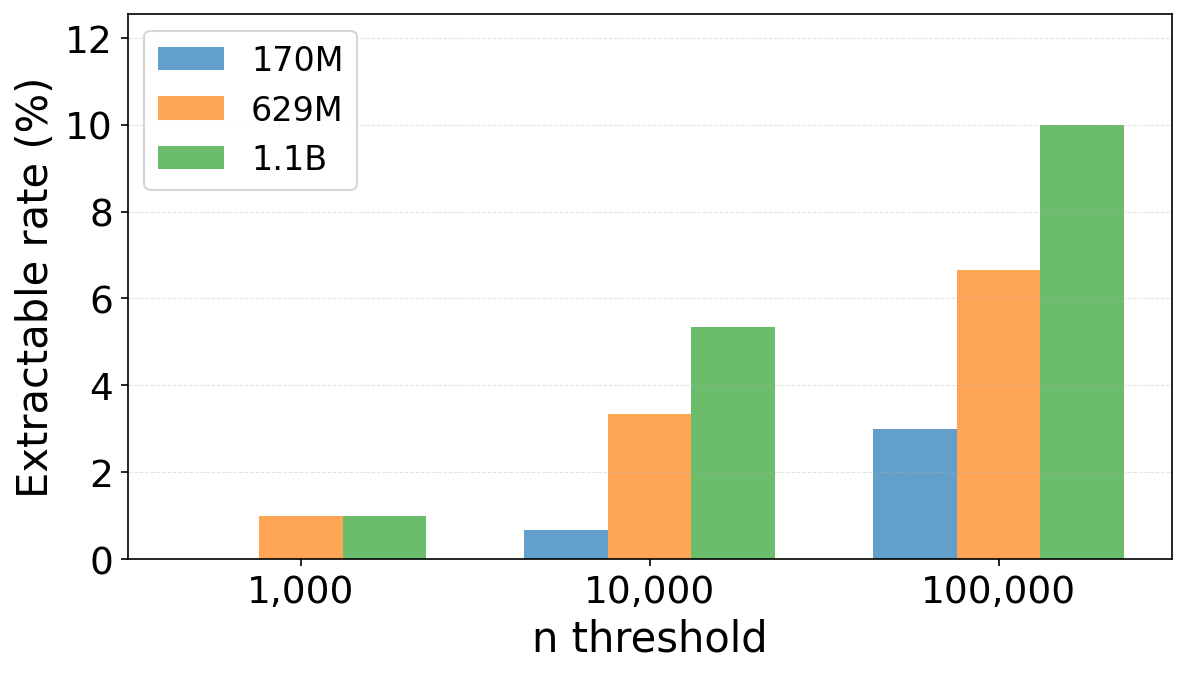}
    \caption{Per-token-step email hit rates under different model scales (170M/629M/1.1B). Each bar reports the Extractable rate over $3{,}000$ samples, measured as whether an email can be recovered within a querying budget $n$, at fixed $p=0.5$.}
    \label{fig:differentscale}
\end{figure}

Figure~\ref{fig:differentscale} reports $(n,p)$-discoverable email memorization across aligned DLMs scales, with extractable rates increasing with model size. For each model, gains from larger query budgets are sublinear, suggesting most practically discoverable memorization is recovered with modest $n$ and further increases yield little additional privacy leakage. Detailed PII leakage for other DLMs scales is provided in Appendix~\ref{app:pii}.

\subsection{Validating Memorization Beyond Generalization}\label{sec:enrontrec}
\begin{figure}[ht]
    \centering
    \includegraphics[width=0.7\linewidth]{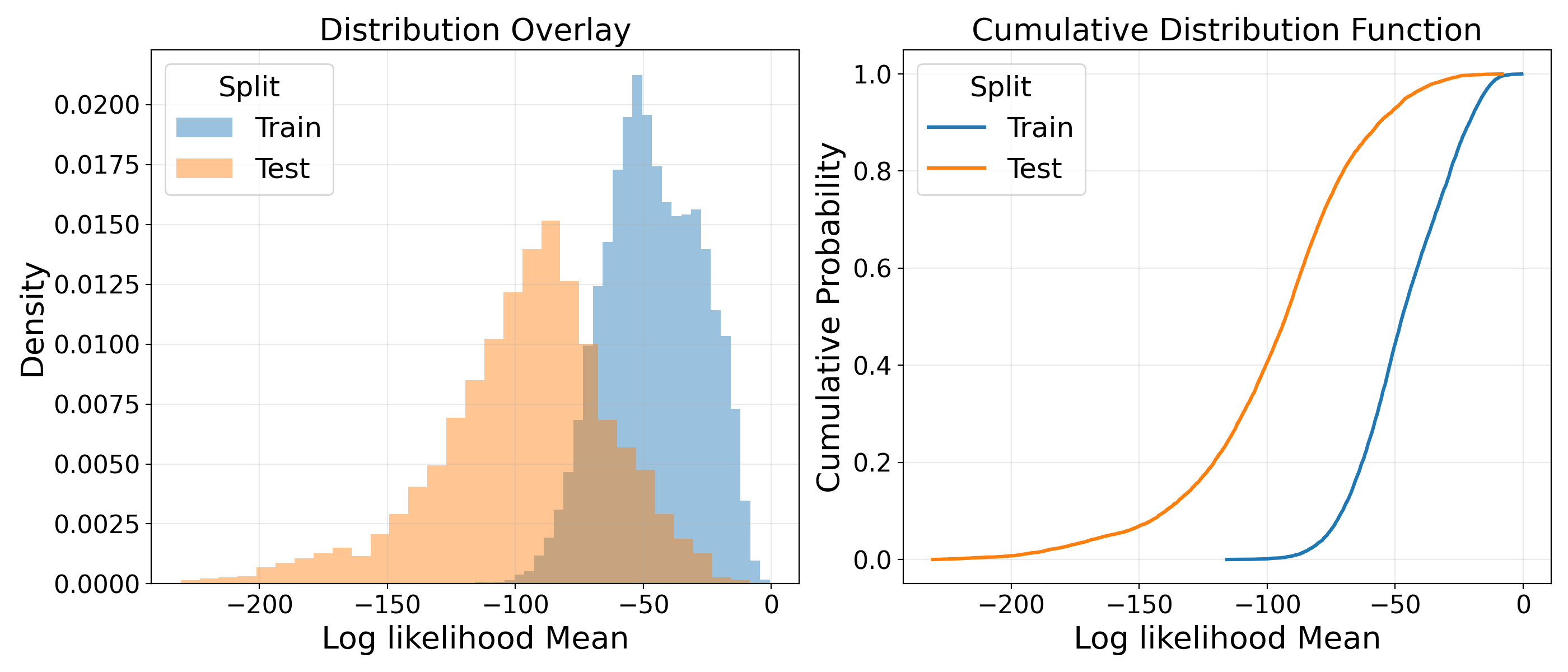}
    \caption{Validating memorization of training data. We compare the reconstruction likelihood from the Enron (train) against the unseen from the TREC 2007 Spam (test), which is from the same domain.}
    \label{fig:traintest}
\end{figure}

A key question is whether our extraction metric reflects true training-data memorization rather than generalization-based reconstruction, especially since \textsc{DLMs-1.1B} shows a low level of verbatim memorization. 

To validate the metric, we follow prior work \cite{hayes2025measuring} and use the TREC 2007 Spam email dataset\cite{bratko2006spam}. Like the Enron dataset, it consists of email text and therefore lies in the same broad domain, while remaining disjoint from the training data. From each corpus, we sample $10,000$ contents of length $100$ tokens, and evaluate extraction using a mask ratio of $0.25$ with $512$ sampling queries per example.

Figure \ref{fig:traintest} compares training and test data under our extraction metric using a distribution overlay and the corresponding cumulative distribution function. In both views, the training data distribution is shifted toward higher log-likelihood values than the unseen test data, indicating uniformly higher reconstruction likelihood for training data. This clear separation suggests that the metric captures memorization rather than reconstruction driven by generalization. Accordingly, even though \textsc{DLM-1.1B} shows low overall verbatim memorization, it exhibits a systematic advantage on Enron, implying that the extracted content reflects memorized training examples rather than coincidental matches. The consistent gap between train and test further supports the robustness of the previous section’s findings. Additional results for the other model scales are provided in Appendix~\ref{app:pii}, Figure~\ref{fig:traintest-appendix}.

\section{Conclusion and Future Work}
This paper introduces a generalized probabilistic extraction framework for measuring memorization in diffusion language models, allowing discoverable extraction beyond prefix-based decoding to arbitrary masking patterns and stochastic diffusion sampling trajectories. It further establishes a principled link between sampling design and leakage risk by proving that finer-grained sampling resolution (more denoising steps) increases the probability of memorization, with autoregressive decoding recoverable as a special case of maximal resolution. Empirically, across model scales and sampling strategies, and under an aligned prefix-conditioned PII completion task, DLMs exhibit lower memorization than ARM counterparts.  Building on these results, future work will measure how post-training (SFT and preference optimization) changes memorization in diffusion language models using evaluation protocols that preserve aligned behavior and match real-deployment prompting. In particular, it is important to test whether post-training drives DLMs toward more prefix–suffix-style generation, and whether this shift increases verbatim memorization and extractability.

\section*{Impact Statements}






This paper presents a theoretical and empirical analysis of memorization in diffusion-based Large Language Models. Our findings contribute to the broader field of trustworthy AI by elucidating how sampling strategies influence the risk of training data leakage. While our analysis reveals potential vulnerabilities in how these models retain training data, this understanding is a prerequisite for building robust defenses and ensuring the responsible deployment of generative models in privacy-sensitive domains. There are many potential societal consequences of our work, and we believe our contribution to data transparency and model auditing is the most significant one.

\section*{Acknowledgements}
We thank the Aalborg University AI:X initiative for enabling this work via the AI:SECURITY lab. We further acknowledge the support of the AAU AI Cloud and express our gratitude to DeiC for providing computing resources on the LUMI cluster (project nr. 465002249).

\bibliography{ref}
\bibliographystyle{unsrtnat}

\newpage
\appendix
\onecolumn

\section{Proof of Theorem~\ref{theo.sampling}}\label{app.theo1proof}

\begin{proof}

Consider two sampling step counts $N_1 > N_2$, where the partition $\mathcal{P}_{N_1}$ is a refinement of $\mathcal{P}_{N_2}$. For any token $\pi \in \mathcal{M}$, let $\mathcal{U}^{(prev)}$ and $\mathcal{W}^{(prev)}$ be the observed index sets at the step immediately preceding the recovery of $\pi$ in the respective $N_1$ and $N_2$ steps sampling processes.

Since $\mathcal{P}_{N_1}$ is a refinement of $\mathcal{P}_{N_2}$ and the recovery order $\sigma$ is fixed, the $N_1$ process updates the observed set more frequently. Specifically, for any token $\pi$, the set of tokens already recovered in the $N_1$ process is always a superset of those recovered in the $N_2$ process at the time $\pi$ is sampled:
\begin{equation}
    \mathcal{W}^{(prev)} \subseteq \mathcal{U}^{(prev)}.
\end{equation}

By Assumption~\ref{ass.mono}, the probability of correctly recovering a token $\pi$ increases with the size of the observed set. Therefore, for every $\pi \in \mathcal{M}$:
\begin{equation}
    \text{Pr}(\hat{z}_{\pi} = z_{\pi} \mid \v{z}_{\mathcal{U}^{(prev)}}) \geq \text{Pr}(\hat{z}_{\pi} = z_{\pi} \mid \v{z}_{\mathcal{W}^{(prev)}}).
\end{equation}
Since we assume the same recovery order of tokens in two cases, taking the product over all $\pi \in \mathcal{M}$, we can conclude that:
\begin{equation}
    \text{Pr}(\hat{\v{z}}_{\mathcal{M}} = \v{z}_{\mathcal{M}} \mid \v{z}_{\bar{\mathcal{M}}}, N_1) \geq \text{Pr}(\hat{\v{z}}_{\mathcal{M}} = \v{z}_{\mathcal{M}} \mid \v{z}_{\bar{\mathcal{M}}}, N_2).
\end{equation}
This completes the proof.
\end{proof}

\section{Validation of generalized $(\epsilon,n,p)$-discoverable extraction}\label{app.enp}

We initiate our empirical analysis by performing $R=10,000$ stochastic generation trials on a series of infilling tasks, where 20 tokens are randomly masked within a 100-token window, randomly selecting from Sec~\ref{sec:6.2} masked-pattern. For each trial, we record the exact number of wrongly reconstructed tokens. Based on the empirical observations of the resulting distribution, we postulate the following assumptions.

\begin{assumption}
    For a diffusion-based LLM, let $d(\hat{\v z}^{(i)}_{\mathcal{M}}, \v z_{\mathcal{M}})$ denote the Hamming distance between the $i$-th independently sampled reconstruction and the ground-truth suffix.
    We observe that for sufficiently large $|\mathcal{M}|$, the distribution of $d$ over multiple trials exhibits a strong central tendency. Specifically, within its central regime (excluding the tails), the distribution can be approximated by a normal distribution
\begin{equation}
    d(\hat{\v z}_{\mathcal{M}}^{(i)}, \v z_{\mathcal{M}}) \approx \mathcal{N}(\mu, \sigma^2),
\end{equation}
where the mean $\mu$ and variance $\sigma^2$ are parameters governed by the model's intrinsic extraction capability.
\end{assumption}

This Gaussian behavior is motivated by the fact that the Hamming distance can be decomposed into a sum of $M$ indicator random variables
\begin{equation}
    d = \sum_{j=1}^{|\mathcal{M}|} \mathbb{I}(\hat{z}_{\pi_j} \neq z_{\pi_j}),
\end{equation}
where each indicator represents a token-level recovery failure.
Owing to the inherent stochasticity of the sampling process in large DLMs, the total error count can be modeled as a sum of weakly dependent Bernoulli random variables. In accordance with the Central Limit Theorem (CLT) for weakly dependent sequences, the aggregate sum converges toward a normal distribution as the number of masked tokens $|\mathcal{M}|$ increases. The Quantile–Quantile (Q-Q) plots in Figure~\ref{fig:qq_cdf_pairs} (left) provide strong empirical evidence for the normality of the recovered token counts. Specifically, the strict adherence of the sample quantiles to the diagonal reference line suggests that the error distribution across $10,000$ trials is well-characterized by a normal distribution. While minor deviations appear in the tail portions, the empirical distribution aligns well with the theoretical normal curve in the central regime. 

In contrast, Figure~\ref{fig:ar_hit_token_hist_pairs} reveals a markedly different trend for ARMs. Due to the sequential nature of the ARM sampling process, there exist significantly stronger inter-token dependencies, which prevent the distribution from converging to a normal form, leading to a distinct departure from the normality assumption.

The right column of Figure~\ref{fig:qq_cdf_pairs} compares the empirical CDF obtained via \eqref{eq.enp_estimation} with the Gaussian CDFs fitted using $10,000$ and $128$ trials, respectively. As illustrated, the Gaussian approximation aligns well with the empirical CDF, effectively capturing the behavior of $\hat{p}_{\mathbf{z}, \epsilon}$. While minor discrepancies still appear in the extreme regions, where the Gaussian model tends to underestimate the probability of tail events far from the mean, the overall fit remains robust. Notably, the curves fitted with $128$ and $10,000$ trials are remarkably similar. This suggests that a significantly smaller number of trials can be employed to obtain a rough estimation of $\hat{p}_{\mathbf{z}, \epsilon}$ within the central regime, offering a more query-efficient alternative to the empirical counting method
\begin{equation}
    \hat{p}_{\mathbf{z}, \epsilon} = \text{Pr}(d(\hat{\mathbf{z}}_{\mathcal{M}}^{(i)}, \mathbf{z}_{\mathcal{M}}) \leq \epsilon) \approx \Phi\left(\frac{\epsilon - \mu}{\sigma}\right),
\end{equation}
where $\Phi(\cdot)$ denotes the CDF of the standard normal distribution $\mathcal{N}(0,1)$, while $\mu$ and $\sigma$ represent the mean and standard deviation estimated from a few trials.

\begin{figure}[ht]
  \centering

  \begin{subfigure}[t]{0.4\textwidth}
    \centering
    \includegraphics[width=\linewidth]{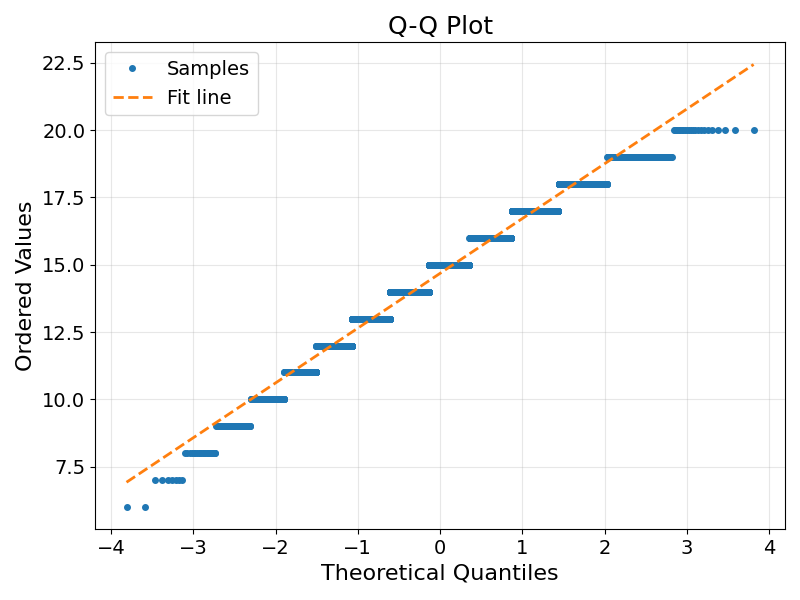}
  \end{subfigure}
  \begin{subfigure}[t]{0.4\textwidth}
    \centering
    \includegraphics[width=\linewidth]{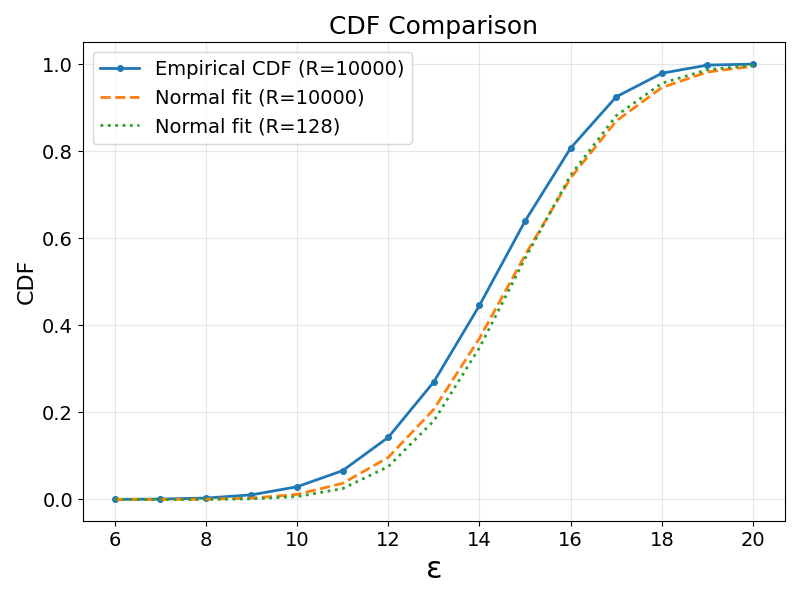}
  \end{subfigure}

  \vspace{0.6em}

  \begin{subfigure}[t]{0.4\textwidth}
    \centering
    \includegraphics[width=\linewidth]{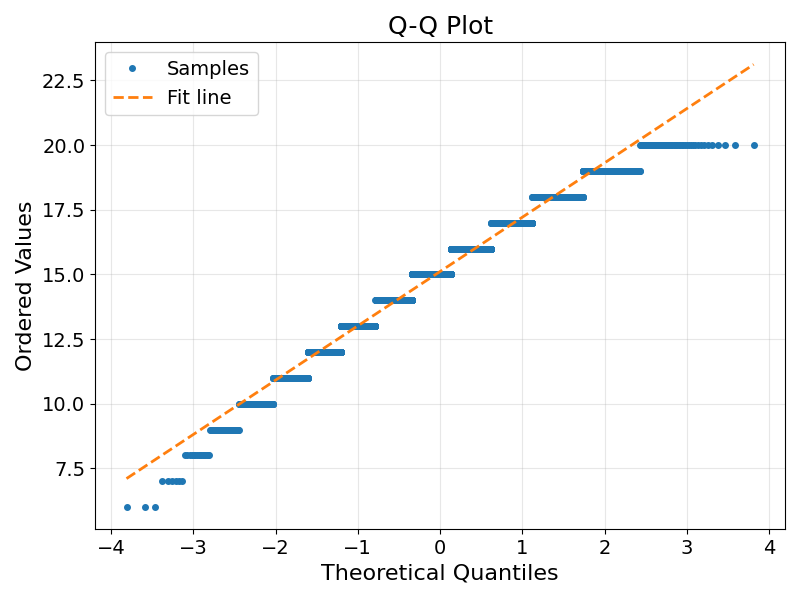}
  \end{subfigure}
  \begin{subfigure}[t]{0.4\textwidth}
    \centering
    \includegraphics[width=\linewidth]{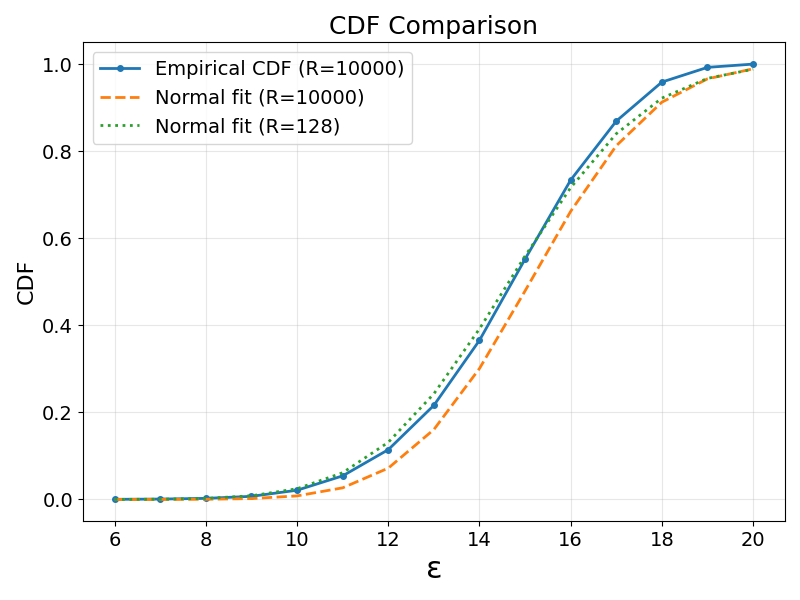}
  \end{subfigure}

  \vspace{0.6em}

  \begin{subfigure}[t]{0.4\textwidth}
    \centering
    \includegraphics[width=\linewidth]{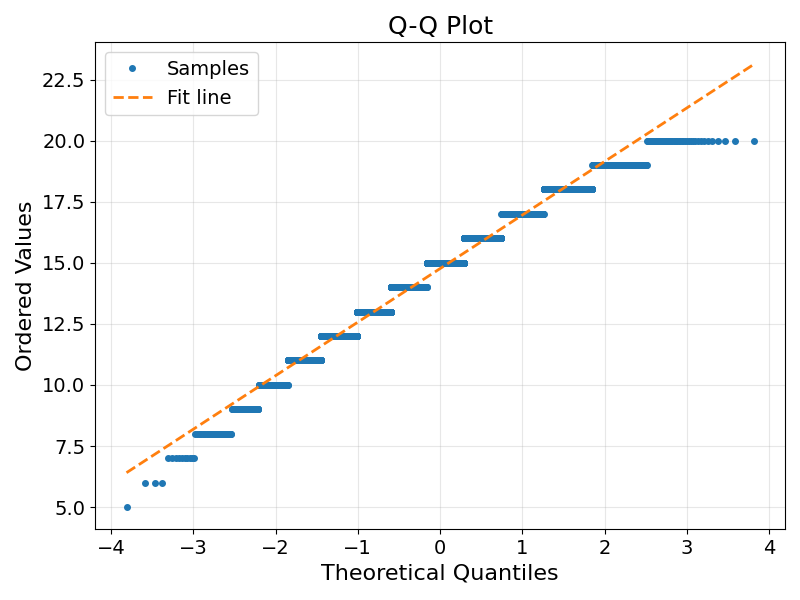}
  \end{subfigure}
  \begin{subfigure}[t]{0.4\textwidth}
    \centering
    \includegraphics[width=\linewidth]{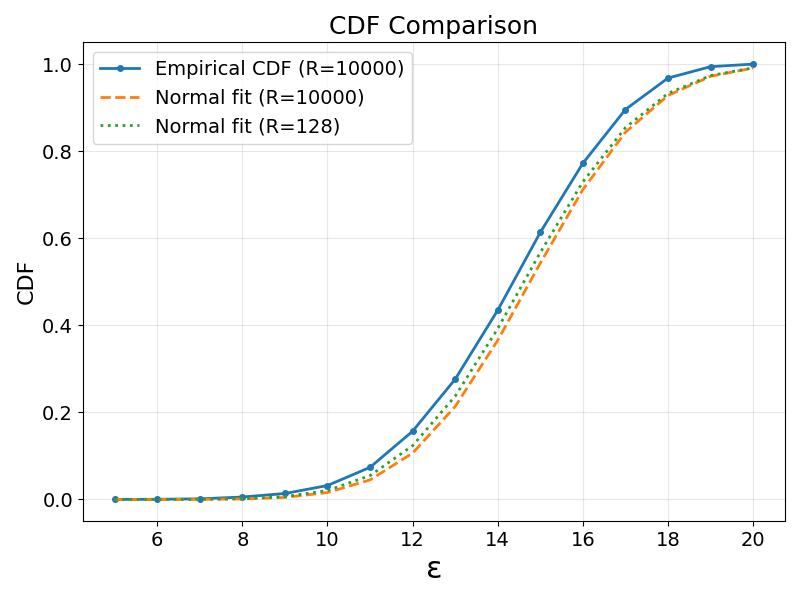}
  \end{subfigure}

  \caption{Paired Q-Q plots (left) and CDF comparisons (right) for three cases for DLM.}
  \label{fig:qq_cdf_pairs}
\end{figure}

\begin{figure}[ht]
  \centering

  \begin{subfigure}[t]{0.4\textwidth}
    \centering
    \includegraphics[width=\linewidth]{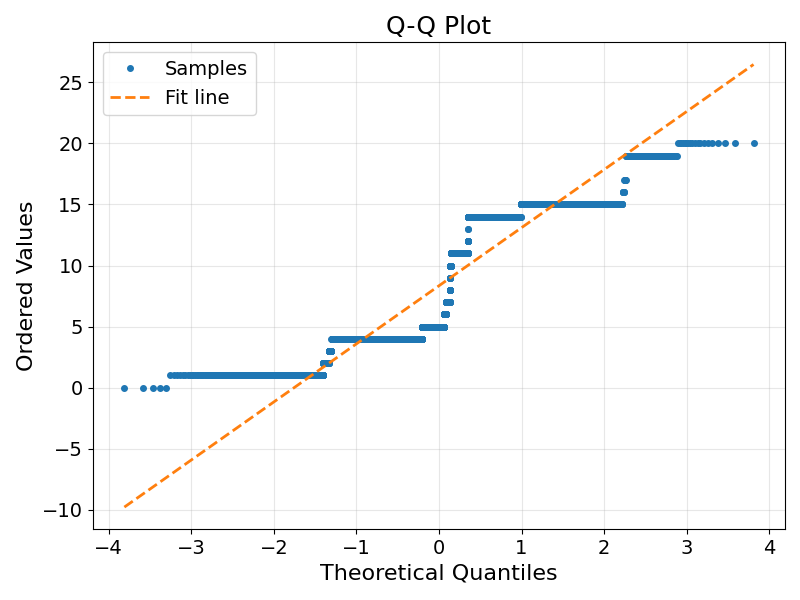}
  \end{subfigure}
  \begin{subfigure}[t]{0.4\textwidth}
    \centering
    \includegraphics[width=\linewidth]{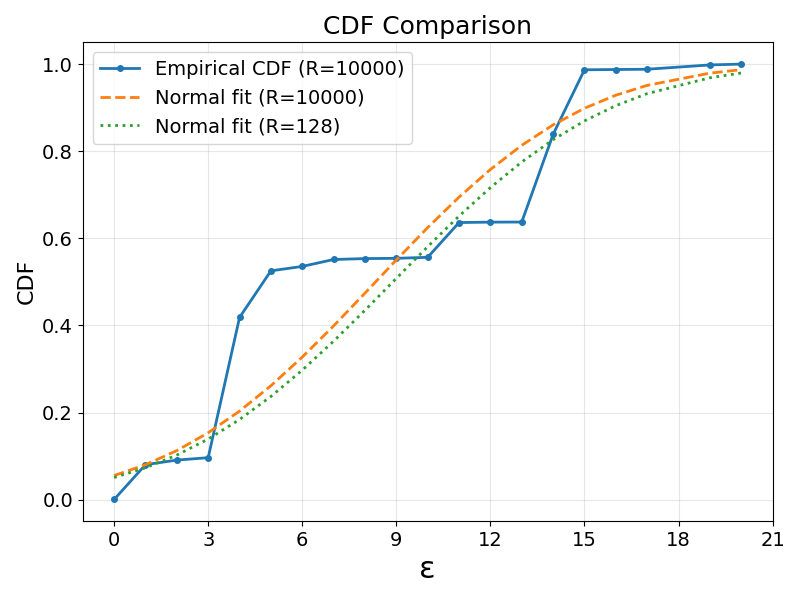}
  \end{subfigure}

  \vspace{0.6em}

  \begin{subfigure}[t]{0.4\textwidth}
    \centering
    \includegraphics[width=\linewidth]{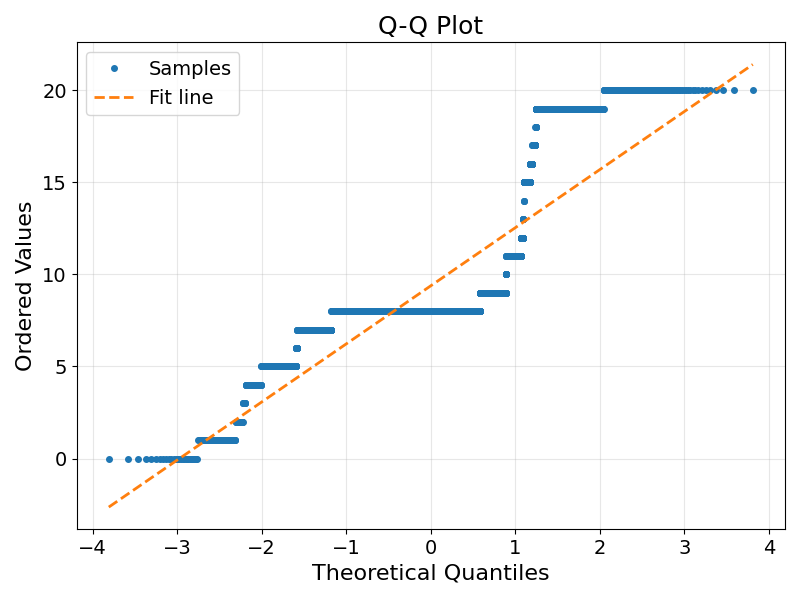}
  \end{subfigure}
  \begin{subfigure}[t]{0.4\textwidth}
    \centering
    \includegraphics[width=\linewidth]{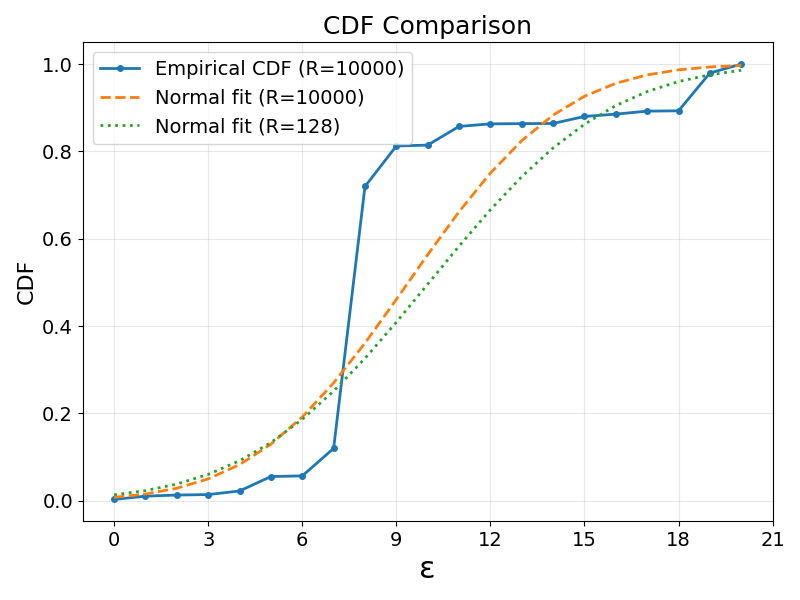}
  \end{subfigure}

  \vspace{0.6em}

  \begin{subfigure}[t]{0.4\textwidth}
    \centering
    \includegraphics[width=\linewidth]{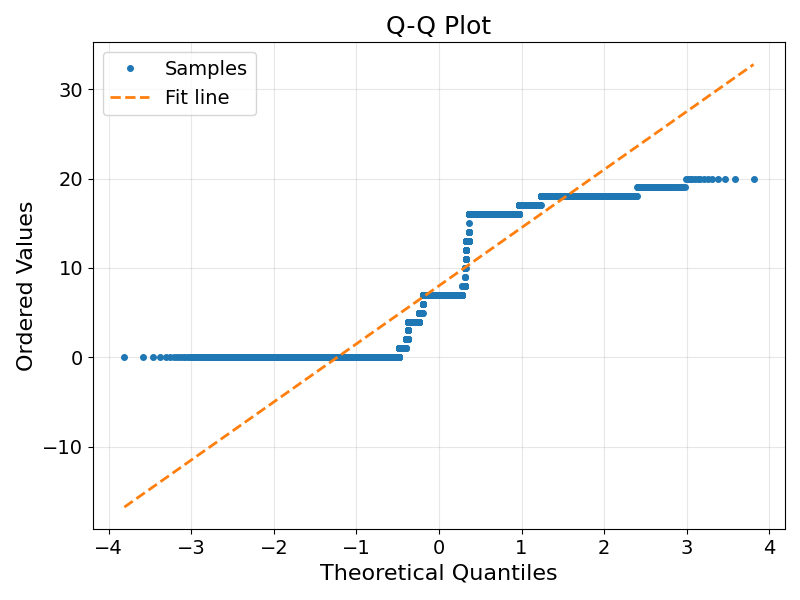}
  \end{subfigure}
  \begin{subfigure}[t]{0.4\textwidth}
    \centering
    \includegraphics[width=\linewidth]{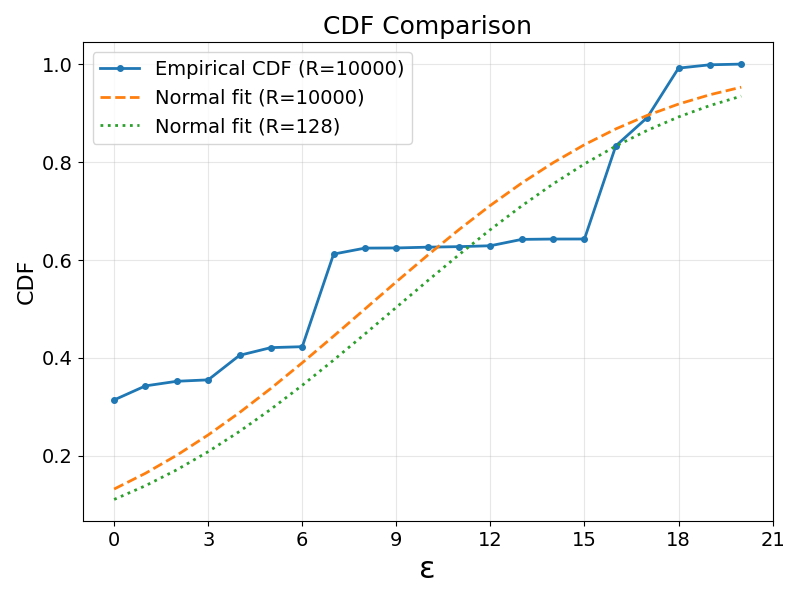}
  \end{subfigure}

  \caption{Paired Q-Q plots (left) and CDF comparisons (right) for three cases for ARM.}
  \label{fig:ar_hit_token_hist_pairs}
\end{figure}


\section{Effect of Masked-patterns on Memorization Estimation}
\label{appendix:sampling-resolution}

\begin{figure}[ht]
    \centering
    \begin{subfigure}{0.6\textwidth}
        \centering
        \includegraphics[width=0.9\linewidth]{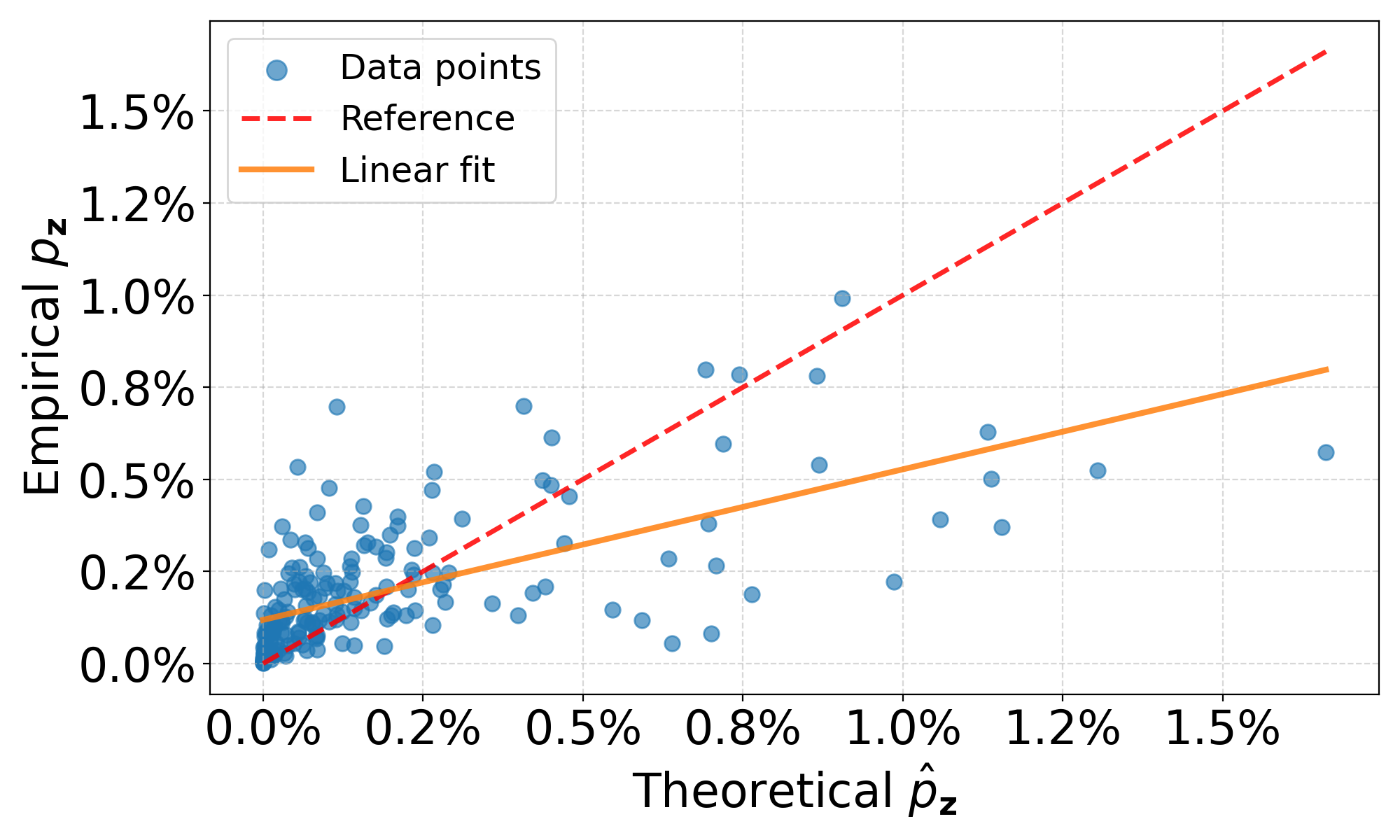}
        \caption{masked-patterns = 128}
        \label{fig:scatter-128}
    \end{subfigure}

    \vspace{0.5em}

    \begin{subfigure}{0.6\textwidth}
        \centering
        \includegraphics[width=0.9\linewidth]{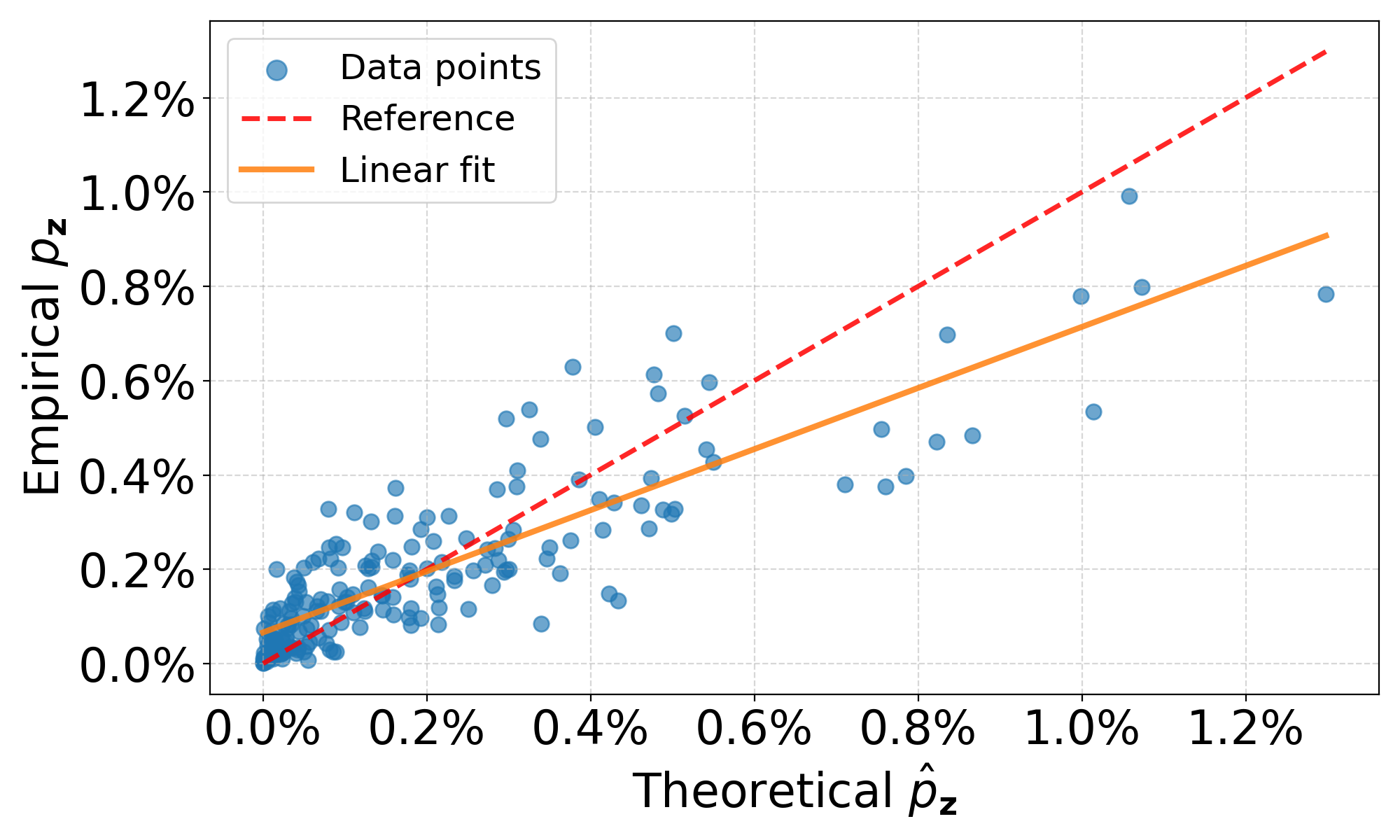}
        \caption{masked-patterns = 256}
        \label{fig:scatter-256}
    \end{subfigure}

    \vspace{0.5em}

    \begin{subfigure}{0.6\textwidth}
        \centering
        \includegraphics[width=0.9\linewidth]{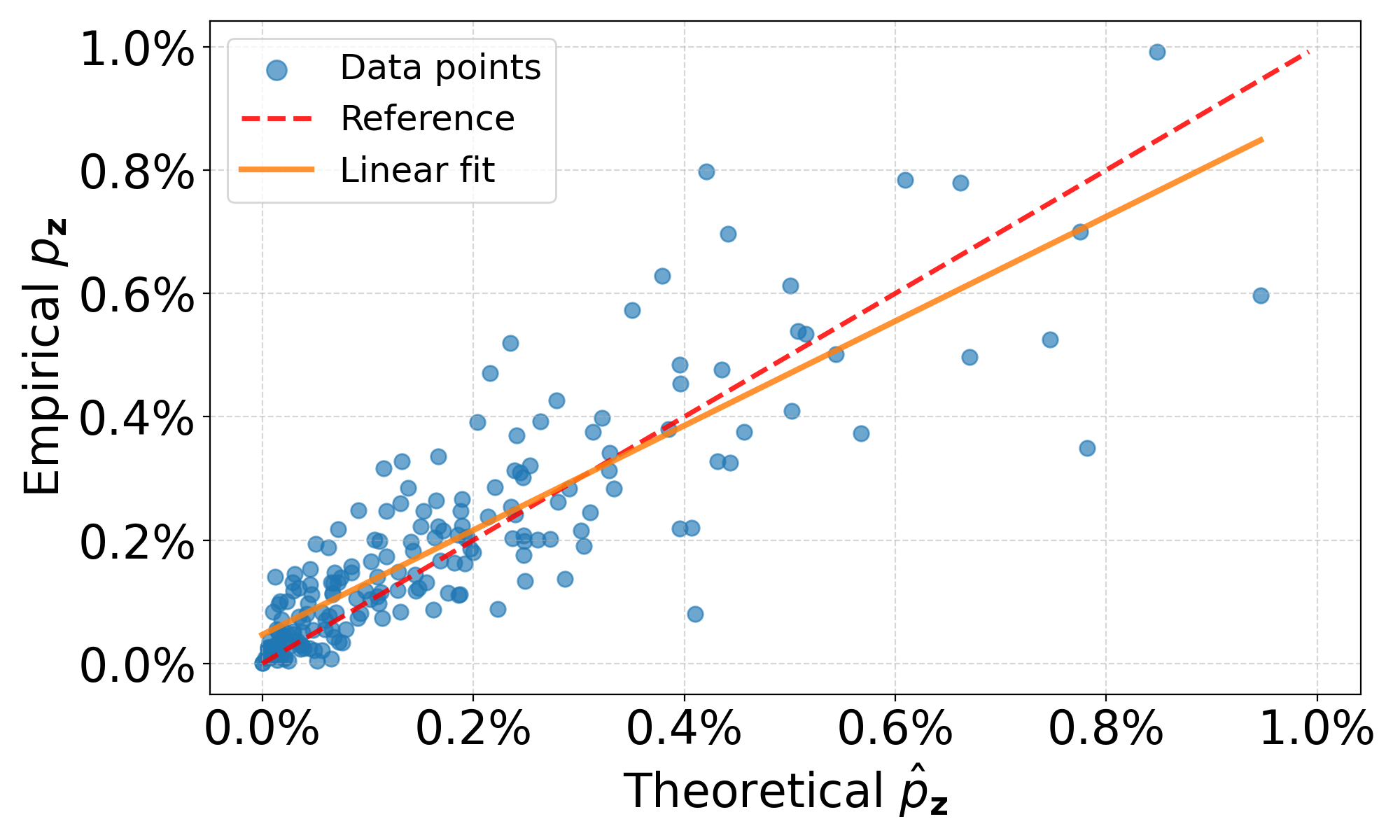}
        \caption{masked-patterns = 512}
        \label{fig:scatter-512}
    \end{subfigure}

    \caption{Comparison of empirical and theoretical memorization probabilities under different sampling masked-patterns quantities. As the number of samples increases, the empirical values converge towards the theoretical identity line.}
    \label{fig:scatter-appendix}
\end{figure}
In this section, we provide additional empirical evidence regarding the impact of the number of sampled masked-patterns $R$ on the estimation of memorization probabilities. Figure~\ref{fig:scatter-appendix} illustrates the correlation between the empirical memorization probability ($p_z$) and the theoretical prediction ($\hat{p}_z$) across different masked-patterns.

\textbf{Convergence and Accuracy.} Our results demonstrate that the estimation accuracy significantly improves as the number of sampled masked-patterns $R$ increases from 128 to 512. We observe that sampling $R=512$ times provides a high-fidelity fit to the theoretical line, suggesting that this amount of masked-patterns is sufficient to capture the underlying memorization behavior of DLMs without prohibitive computational overhead.

\textbf{Analysis of Empirical Variance.} We observe that the empirical points exhibit a degree of dispersion around the identity line across the entire probability spectrum . This variance is a natural consequence of the stochastic nature of the denoising process in DLMs:
\begin{itemize}[leftmargin=*]
    \item \textbf{Stochastic Sampling Fluctuations:} Since each empirical data point is calculated as the mean of $R$ random masked, it is essentially a realization of a binomial-like estimator. Therefore, some level of statistical fluctuation is inevitable. Even for samples with high theoretical memorization probability, the specific set of $R$ trajectories sampled may include more or fewer "successful" extraction paths than the expected average.
    \item \textbf{Natural Statistical Dispersion:} This variance does not represent a flaw in the theoretical derivation, but rather reflects the inherent randomness of the sampling-based estimation. Crucially, as $R$ increases, the standard deviation of these fluctuations decreases, and the empirical values remain centered around the theoretical identity line within a consistent order of magnitude.
\end{itemize}

In summary, the empirical results in Figure~\ref{fig:scatter-appendix} confirm that our theoretical framework provides an accurate "mean" behavior for data extraction. The observed variance is consistent with standard Monte Carlo estimation errors and does not impede the practical utility of our metric in identifying memorized content.

\section{PII Memorization}\label{app:pii}
\begin{table}[t]
\centering
\small
\setlength{\tabcolsep}{4pt}
\begin{tabular}{llcccccccc}
\toprule
\multirow{2}{*}{\textbf{Model}} & \multirow{2}{*}{\textbf{Step}} &
\multicolumn{4}{c}{\textbf{Email}} &
\multicolumn{4}{c}{\textbf{Phone}} \\
\cmidrule(lr){3-6}\cmidrule(lr){7-10}
 
& & $p$ = 10\% & $p$ = 50\% & $p$ = 90\% & $p$ = 99\%
  & $p$ = 10\% & $p$ = 50\% & $p$ = 90\% & $p$ = 99\% \\
\midrule
\multirow{2}{*}{\textsc{DLM-170M}}
  & One  & 0  & 0  & 0 & 0  & 0  & 0  & 0 & 0 \\
  & Max  & 7 & 2 & 0 & 0  & 0  & 0  & 0 & 0 \\
\cmidrule(lr){1-10}
\multirow{2}{*}{\textsc{DLM-629M}}
  & One  & 0  & 0  & 0 & 0  & 0  & 0  & 0 & 0 \\
  & Max  & 17 & 10 & 7 & 4  & 0  & 0  & 0 & 0 \\
\cmidrule(lr){1-10}
\multirow{2}{*}{\textsc{DLM-1.1B}}
  & One  & 0  & 0  & 0 & 0  & 0  & 0  & 0 & 0 \\
  & Max  & 26 & 16 & 10 & 7  & 0  & 0  & 0 & 0 \\
\bottomrule
\end{tabular}
\caption{Generalized $(n,p)$-discoverable Extraction. Number of memorized samples at target probabilities $p$ under a query budget of $n=10{,}000$ on $3{,}000$ completion prompts, for different PII types.}
\label{tab:nptableappendix}
\end{table}

\begin{figure}[t]
    \centering
    \begin{subfigure}{0.9\textwidth}
        \centering
        \includegraphics[width=\linewidth]{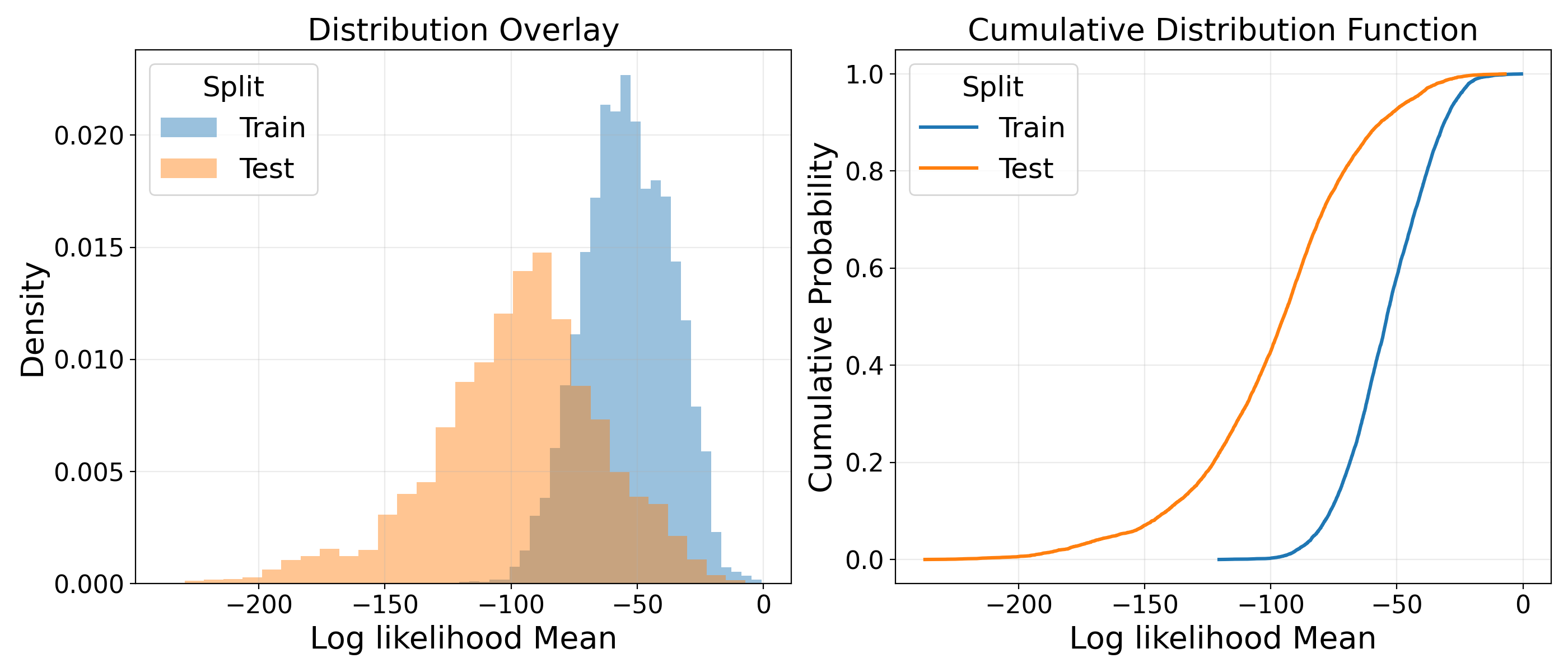}
        \caption{\textsc{DLM-170M}}
        \label{fig:169}
    \end{subfigure}

    \vspace{0.5em}

    \begin{subfigure}{0.9\textwidth}
        \centering
        \includegraphics[width=\linewidth]{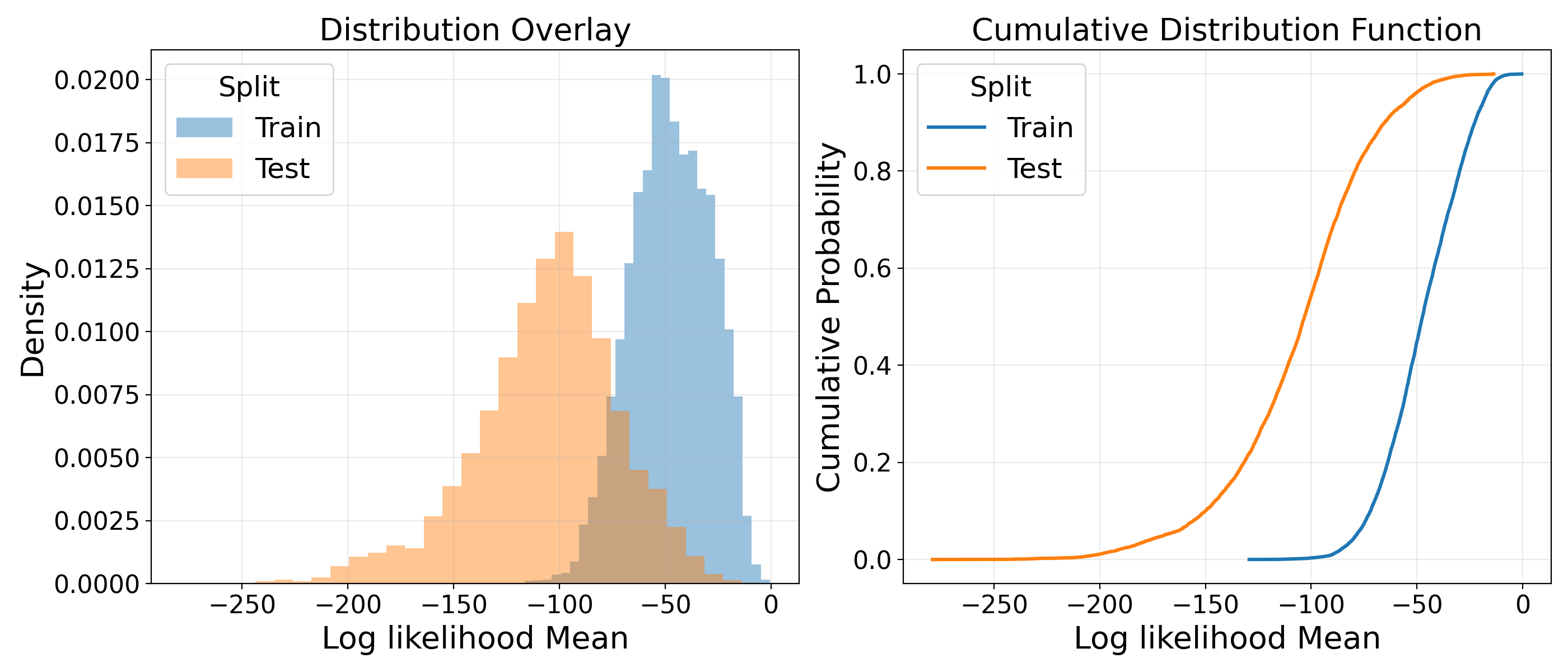}
        \caption{\textsc{DLM-629M}}
        \label{fig:512}
    \end{subfigure}

    \caption{Validating memorization of training data under different DLM scales. We compare the reconstruction likelihood from the Enron (train) against the unseen from the TREC 2007 Spam (test), which is from the same domain.}
    \label{fig:traintest-appendix}
\end{figure}
In this section, we extend our evaluation of $(n, p)$-discoverable memorization to smaller model variants (\textsc{DLM-170M} and \textsc{DLM-629M}) across different types of Personally Identifiable Information (PII). Table~\ref{tab:nptableappendix} summarizes the number of memorized samples for email addresses and phone numbers under varying target probability thresholds $p$.

\textbf{Low Leakage in Small-Scale DLMs.} Our empirical results indicate that small-scale DLMs exhibit extremely low leakage probabilities for PII. This is particularly evident for \textit{phone numbers}, where the number of memorized samples remains zero across all tested models and confidence levels $p \in \{10\%, 50\%, 90\%, 99\%\}$. Even for email addresses, which appear slightly more susceptible to extraction, the risk remains significantly lower compared to larger model scales.

\textbf{Validation of Model Memorization.} To ensure that these low extraction rates are a meaningful reflection of the model's privacy-preserving properties—rather than a byproduct of underfitting or a general failure to memorize the training set—we conducted auxiliary measurements. Specifically, we evaluated the \textsc{DLM-170M} and \textsc{DLM-629M} models using the same training data reconstruction methodology described in Sec~\ref{sec:enrontrec}. Figure \ref{fig:traintest-appendix} confirmed that both models show higher confidence in the Enron training corpus. Therefore, the low extraction rates reported in Table~\ref{tab:nptableappendix} serve as a valuable and reliable reference for PII risks in smaller diffusion language models.

\section{Memorization examples}\label{app:memexp}
We include the following memorized examples from Section~\ref{sec:6.2} in this appendix.

\begin{tcolorbox}[
  colback=gray!5!white,
  colframe=gray!50!black,
  boxrule=0.5pt,
  arc=2pt,
  left=2pt,right=2pt,top=2pt,bottom=2pt
]
\scriptsize
\textbf{Original}

\begin{Verbatim}[
  breaklines=true,
  breakanywhere=true,
  breaksymbolleft=,
  breaksymbolright=
]
1948 Summer Olympics.

References

External links
 

1927 births
2020 deaths
Athletes (track and field) at the 1948 Summer Olympics
Brazilian female sprinters
Olympic athletes of Brazil
Olympic female sprinters</s>  Why is Tommy Hunter making us all nervous?
One of the Phillies' offseason acquisitions
\end{Verbatim}

\textbf{Recoverable Masked-Pattern}

\begin{Verbatim}[
  breaklines=true,
  breakanywhere=true,
  breaksymbolleft=,
  breaksymbolright=
]
<MASK>  <MASK> 948 <MASK>  <MASK> . <MASK> 
<MASK> ferences

External <MASK> 
 

<MASK> 92 <MASK>  <MASK> s <MASK>  <MASK> 020 deaths <MASK> Athlet <MASK>  (track and <MASK> ) at the 19 <MASK>  <MASK>  Summer <MASK>  <MASK> Brazilian female sprinter <MASK> 
Olympic <MASK> es <MASK>  Brazil <MASK> O <MASK> ic <MASK>  spr <MASK> s</s>  Why <MASK>  Tommy Hunter making us all nervous?
One <MASK>  <MASK>  Phillies' offseason <MASK> quisitions
\end{Verbatim}

\end{tcolorbox}

\begin{tcolorbox}[
  colback=gray!5!white,
  colframe=gray!50!black,
  boxrule=0.5pt,
  arc=2pt,
  left=2pt,right=2pt,top=2pt,bottom=2pt
]
\scriptsize
\textbf{Original}

\begin{Verbatim}[
  breaklines=true,
  breakanywhere=true,
  breaksymbolleft=,
  breaksymbolright=
]
te change is a natural phenomenon caused by several factors such as solar output, the planet's orbit, and ocean circulation. However, human activity that releases greenhouse gases into the atmosphere has made the effects of climate change more severe.
According to scientists, even if the world stopped emitting carbon dioxide today, global warming would persist for decades to come because of the amount of carbon dioxide we've already pumped into the atmosphere.
\end{Verbatim}

\textbf{Recoverable Masked-Pattern}

\begin{Verbatim}[
  breaklines=true,
  breakanywhere=true,
  breaksymbolleft=,
  breaksymbolright=
]
te change is a natural <MASK>  <MASK>  caused <MASK>  several factors such <MASK>  solar output <MASK>  <MASK>  planet <MASK> s orbit, <MASK>  ocean circulation. However, human activity that releases green <MASK>  gases <MASK>  the atmosphere <MASK>  made <MASK>  effects of climate change more severe.
According <MASK>  scientists, even <MASK>  the <MASK>  <MASK>  emitting carbon dioxide today, global war <MASK>  would persist for decades to come because of the amount of carbon diox <MASK>  we've already <MASK> ump <MASK>  into the atmosphere.
\end{Verbatim}

\end{tcolorbox}

\begin{tcolorbox}[
  colback=gray!5!white,
  colframe=gray!50!black,
  boxrule=0.5pt,
  arc=2pt,
  left=2pt,right=2pt,top=2pt,bottom=2pt
]
\scriptsize
\textbf{Original}

\begin{Verbatim}[
  breaklines=true,
  breakanywhere=true,
  breaksymbolleft=,
  breaksymbolright=
]
control. Researchers believe that hereditary ataxia may supply clues to the genetic causes of RLS.
Attention-deficit hyperactivity disorder (ADHD).
Psychiatric disorders, such as depression.
Environmental and Dietary Factors
The following environmental and dietary factors can trigger or worsen RLS:
Iron deficiency. People who are deficient in iron are at risk for R
\end{Verbatim}

\textbf{Recoverable Masked-Pattern}

\begin{Verbatim}[
  breaklines=true,
  breakanywhere=true,
  breaksymbolleft=,
  breaksymbolright=
]
control. Researchers believe that her <MASK> ary ataxia may supply clues to the genetic causes of RLS.
Att <MASK>  <MASK>  <MASK> ic <MASK>  hyperactivity disorder <MASK> AD <MASK> ).
<MASK> sychiatric disorders, such as de <MASK>  <MASK> 
<MASK> al and <MASK> ary Factors
The following environmental and <MASK> etary <MASK>  can trigger or <MASK> orsen RLS <MASK> 
I <MASK>  deficiency. People who <MASK>  defic <MASK>  <MASK>  iron are at risk for R
\end{Verbatim}

\end{tcolorbox}

\clearpage
\section{PII extraction}\label{app:regex}

We adopt the regular expressions in \cite{kim2023propile}.\ to search for email addresses and U.S.\ phone numbers; the patterns are shown below.

\begin{lstlisting}[language=Python, basicstyle=\ttfamily\small, frame=single, breaklines=true]
EMAIL_RE = re.compile("^([a-zA-Z0-9_\-\.]+)@([a-zA-Z0-9_\-\.]+)\.([a-zA-Z]{2,5})$")
PHONE_RE = re.compile("[0-9][0-9][0-9][-.()][0-9][0-9][0-9][-.()][0-9][0-9][0-9][0-9]")
\end{lstlisting}

\end{document}